\documentclass[11pt]{article} 

\usepackage{amssymb,latexsym,amsmath}
\usepackage{epsfig}
\usepackage{caption}
 
 \usepackage{amsmath}
\usepackage{amsfonts}
\usepackage{array}
\usepackage{dsfont}
\usepackage{amssymb}
\usepackage{amsthm}

\newtheorem{assumption}{Assumption}

 \newtheorem{remark}{Remark}
\newtheorem{theorem}{Theorem}
\newtheorem{prop}{Proposition}
\newtheorem{corollary}{Corollary}

\newtheorem{lemma}{Lemma}

\usepackage[colorlinks=true,breaklinks=true,bookmarks=true,urlcolor=blue,
     citecolor=blue,linkcolor=blue,bookmarksopen=false,draft=false]{hyperref}

\def\Pr{\mathop{\rm Pr}}

\def\argmin{\mathop{\rm arg\, min}}

\def\B{{\mathcal B}}

\def\P{{\mathcal P}}

\newcommand{\R}{\mathds{R}}
\newcommand{\Zplus}{\mathbb{Z}_+}

\usepackage[utf8]{inputenc}

\begin{document}

\sloppy
\title{Reinforcement Learning with Function Approximation for Non-Markov Processes}
\author{Ali Devran Kara
\thanks{The author is with the Department of Mathematics,
     Florida State University, Tallahassee, FL, USA,
     Email: akara@fsu.edu}
     }
\maketitle

\maketitle
\begin{abstract}
We study reinforcement learning methods with linear function approximation under non-Markov state and cost processes. We first consider the policy evaluation method and show that the algorithm converges under suitable ergodicity conditions on the underlying non-Markov processes. Furthermore, we show that the limit corresponds to the fixed point of a joint operator composed of an orthogonal projection and the Bellman operator of an auxiliary \emph{Markov} decision process. 

For Q-learning with linear function approximation, as in the Markov setting, convergence is not guaranteed in general. We show, however, that for the special case where the basis functions are chosen based on quantization maps, the convergence can be shown under similar ergodicity conditions. Finally, we apply our results to partially observed Markov decision processes, where finite-memory variables are used as state representations, and we derive explicit error bounds for the limits of the resulting learning algorithms.
\end{abstract}

\section{Introduction}

Model-free reinforcement learning methods aim to compute approximately optimal control policies, or the value function of a stochastic control problem, directly from interaction data without constructing a model of the dynamics. Although these algorithms do not require explicit knowledge of the dynamics, their theoretical guarantees  rely on the assumption that the underlying control problem is a Markov decision process (MDP). In practice, this assumption is often idealized, holding only in simulated environments.

In this paper, we study reinforcement learning algorithms when the observed state and cost processes are general stochastic processes that do not form an MDP. We focus on methods with linear function approximation and analyze both their convergence properties and the interpretation of the limits if convergence occurs.

We concentrate on two classical reinforcement learning methods under linear function approximation: policy evaluation and Q-learning. Linear function approximation is one of the simplest schemes for handling high-dimensional state spaces. It is also the most theoretically tractable setting, providing insight into the behavior of learning algorithms under function approximation.

Existing convergence analyses often assume that the state process is Markov and that the cost depends only on the current state and action. Under these assumptions, policy evaluation and Q-learning aim to approximate the value of a given policy and the optimal state-action value function, respectively, within the span of the chosen basis functions. 

When the Markov assumption does not hold, it is not immediately clear how these iterations perform. The main questions we address in this paper are:
\begin{itemize}
\item Do the iterations converge if the processes are not Markov? What are the minimal assumptions required to guarantee convergence?
\item If the iterations converge, what does the limit represent?
\item How well do the limiting values approximate the quantities of interest? In particular, can explicit approximation error bounds be obtained?
\end{itemize}

\subsection{Related Work}

One of the main challenges in the optimality analysis and learning of stochastic control problems is the curse of dimensionality. Function approximation methods are widely used to tackle this issue. In particular, reinforcement learning with linear function approximation has been studied extensively for fully observed Markov control problems.

\cite{tsitsiklis1997analysis} was among the first to analyze linear function approximation for policy evaluation in fully observed MDPs, showing the convergence of TD($\lambda$) methods. However, analyzing the learning of optimal Q-values under linear function approximation is more challenging. In particular, the invariant measure of the exploration policy may differ from that induced by the greedy policy, so the algorithm may fail to converge in general. \cite{melo2008analysis} showed the convergence under a covariance dominance condition relating the feature covariance induced by the greedy policy and that induced by the exploration policy. This condition suggests that the exploration policy should not deviate far from greedy action selection in general settings.

Several other special cases guarantee convergence. First, in the exact representation case, if the optimal Q-value lies in the span of the chosen basis functions, it can be learned exactly. In this case, the composition of the projection mapping and the Bellman operator coincides with the Bellman operator itself, and hence remains a contraction under the uniform norm \cite{Ruszczy2024,jin2023provably}. Second, if the basis functions are orthonormal (e.g., in discretization-based approximations), the projection map is non-expansive not only in the $L_2$ norm but also in the uniform norm, allowing convergence and error analysis without restrictive conditions \cite{kara2023q}.

For general basis functions, Meyn \cite{meyn2024projected} recently showed that although the composition of the projection and Bellman operators is not necessarily a contraction, it admits at least one fixed point if the exploration policy is $\epsilon$-greedy. Furthermore, the parameter iterations remain almost surely bounded.

Function approximation beyond fully observed MDPs  remains relatively less studied.   \cite{cai2022reinforcement} study learning for partially observed MDPs using linear function approximation, assuming that the transition and observation densities are exactly representable by the basis functions. They consider finite-memory variables and impose a restrictive observability condition on the observation model, which ensures invertibility of the observation distributions and allows the Bellman mapping for the finite-memory variables to be parametrized. This condition guarantees that any distribution over observations uniquely determines the hidden state distribution.

Q-learning under non-Markovian settings has been studied in a few works, e.g., \cite{dong2022simple, chandak2022reinforcement, kara2021convergence, kara2024qlearning, sinha2024periodic}. Prior to such recent studies, we note that \cite{singh1994learning} showed the convergence of Q-learning for POMDPs with measurements viewed as state variables which represents a special class of non-Markov dynamics.

\cite{kara2021convergence} analyzed Q-learning for partially observed MDPs with finite-window measurements and demonstrated near-optimality under filter stability conditions. Similarly, \cite{sinha2024periodic} studied Q-learning based on the functions of history for POMDPs and proved convergence under general learning rates.

\cite{dong2022simple} proposed a general RL framework for complex environments with finite variables, allowing infinite past dependence, and assuming stationary transitions under certain regularity conditions. \cite{chandak2022reinforcement} analyzed Q-learning convergence in non-Markovian environments by imposing continuity and measurability conditions on the infinite-dimensional observable history, using an ODE-based approach pioneered in \cite{borkar2000ode}. Finally, \cite{kara2024qlearning} established convergence of tabular Q-learning under ergodicity assumptions for the non-Markov state process, showing that the learned values correspond to an auxiliary MDP,  which allows one to compare the performance of the learned controls against the optimal value. 

In this paper, we extend these results to linear function approximation for general non-Markov state and cost processes under ergodicity conditions. We study both policy evaluation and Q-learning using linear function approximations. For policy evaluation, we show that the convergence holds under ergodicity assumptions. As a special case, we consider the partially observed control problems with finite-memory controllers. We provide upper bounds on the error of the learned value, building on the finite-memory approximation framework
developed in \cite{kara2020near,kara2023convergence}. For Q-learning with linear function approximation, convergence is not guaranteed in general. However, under discretization, the algorithm reduces to tabular Q-learning on the discretized non-Markov state process, allowing us to apply results from \cite{kara2024qlearning}. Furthermore, for POMDPs using discretization-based basis functions, the error analysis of \cite{devran2025near} applies under less restrictive assumptions on the model and exploration policy.

\subsection{Problem Formulation}

We consider three stochastic processes:
\begin{itemize}
\item $S_t$ is an $\mathds{S}$-valued stochastic process representing the state,
\item $C_t$ is a real-valued process representing the cost realizations,
\item $U_t\sim\gamma(\cdot|S_t)$ is the control process generated by some randomized  feedback control function $\gamma:\mathds{S}\to\P(\mathds{U})$.
\end{itemize}
Here, $\mathds{S} \subset \mathds{R}^n$ and $\mathds{U} \subset \mathds{R}^m$ are Borel spaces, for some finite $m,n<\infty$. All processes are defined on a filtered probability space $(\Omega,\mathcal{F},\{\mathcal{F}_t\}_{t\ge 0},\mathbb{P})$ and are adapted to the filtration.

We study two reinforcement learning algorithms applied to these processes: policy evaluation (TD(0)) and Q-learning under linear function approximation. Let $\{\phi^i(s)\}_{i=1}^d$, $\phi^i: \mathds{S} \to \mathds{R}$, be a set of known basis functions, and denote $\mathbf{\Phi}^\intercal := [\phi^1, \dots, \phi^d]$. Policy evaluation tracks parameters $\{\theta_t\}$ given by
\begin{align}\label{main_iter_0}
\theta_{t+1} = \theta_t - \alpha_t \mathbf{\Phi}(S_t) \left[ \theta_t^\intercal \mathbf{\Phi}(S_t) - C_t - \beta \theta_t^\intercal \mathbf{\Phi}(S_{t+1}) \right],
\end{align}
where $0 < \beta < 1$ is the discount factor and $\alpha_t$ is the learning rate.

For Q-learning, the basis functions are extended to the action space: $\{\phi^i(s,u)\}_{i=1}^d$, $\phi^i: \mathds{S} \times \mathds{U} \to \mathds{R}$, and parameters are updated as
\begin{align}\label{q_iter_0}
\theta_{t+1} = \theta_t - \alpha_t \mathbf{\Phi}(S_t,U_t) \Big[ \theta_t^\intercal \mathbf{\Phi}(S_t,U_t) - C_t - \beta \min_v \theta_t^\intercal \mathbf{\Phi}(S_{t+1},v) \Big].
\end{align}

In the standard Markovian setup, the state evolves as $S_{t+1} \sim \mathcal{T}(\cdot|S_t,U_t)$ for a Markov kernel $\mathcal{T}$, and the cost depends only on the current state and action: $C_t = c(S_t,U_t)$ for some $c:\mathds{S\times U}\to\mathds{R}$.  For the Markovian standard setup, the algorithms then aim to approximate 
\begin{align*}
\sum_{t=0}^\infty \beta^t \mathbb{E}_\gamma[c(S_t,U_t)|S_0=s_0], \quad 
\inf_\gamma \sum_{t=0}^\infty \beta^t \mathbb{E}_\gamma[c(S_t,U_t)|S_0=s_0,U_0=u_0],
\end{align*}
on the span of the basis functions $\{\phi^i\}$ where the expectations are with respect to the transition kernel $\mathcal{T}$ and the policy $\gamma(du|s)$.  The first term above represents accumulated infinite horizon expected discounted cost under the policy $\gamma$, which we refer to as  the value of the policy $\gamma$. The second term represents the optimal value that can be achieved if the initial state and action pair is given by some $(s_0,u_0)$, which is also referred to as the optimal Q-value for $(s_0,u_0)$ or the state-action value function.

In this paper, we assume that the processes $S_t, C_t, U_t$ do not necessarily follow the standard Markovian setting. We study sufficient conditions that guarantee convergence of the iterations (\ref{main_iter_0}) and (\ref{q_iter_0}) beyond the Markovian case, and we characterize the limit when convergence occurs. Our main contributions are as follows:

\begin{itemize}
\item \textbf{Policy Evaluation (Section \ref{pol_eval_sec}):} We analyze the convergence of the iterations (\ref{main_iter_0}) and characterize their limit. 
\begin{itemize}
\item In Section \ref{stoc_app_sec}, we prove convergence of a stochastic approximation algorithm for solving a linear equation under non-Markov noise, extending the arguments of \cite{benveniste2012adaptive} via decomposition of the noise using a Poisson equation, where we adapt the arguments to non-Markov processes using proper ergodicity and mixing assumptions.
\item In Section \ref{stat_MDP}, we construct an auxiliary Markov decision process, called the \emph{stationary regime MDP}, corresponding to the stationary behavior of the non-Markov state process $S_t$.
\item In Sections \ref{lin_app_proj} and \ref{pol_eval_conv_sec}, we define an orthogonal projection map for the basis functions $\{\phi^i\}_{i=1}^d$ and a Bellman map for the stationary regime MDP. Using the stochastic approximation result, we show that (\ref{main_iter_0}) converges, and that its limit coincides with the fixed point  of the joint map composed of the projection map and the Bellman map for the stationary regime MDP. In particular, this implies that the iterations under non-Markov processes converge to the same limit as if the iterations were applied to a Markov process generated by the stationary regime MDP.
\item In Section \ref{err_anal_sec}, we analyze the error of the learned value with respect to the value of the policy $\gamma(du|s)$ under the stationary regime MDP.
\end{itemize} 

\item \textbf{Q-Learning (Section \ref{q_learn_sec}):} We study the behavior of the projected Bellman operator for the stationary regime MDP under greedy action selection. As in standard MDPs  (not very surprisingly), Q-learning with linear function approximation generally fails to converge under non-Markov processes, except in special cases: 
(i) the cost function and transition kernel of the stationary regime MDP are perfectly linear in the chosen basis functions, 
(ii)  the feature covariance induced by the greedy policy
is uniformly dominated by that induced by the exploration policy after discounting, or 
(iii) the basis functions are constructed using indicator functions on a discretization of $\mathds{S}$ and $\mathds{U}$.

\item \textbf{Partially Observed MDPs (Section \ref{pomdp_section}):} We apply our framework to POMDPs with finite-memory controllers. For policy evaluation under finite memory, we derive explicit error bounds for the learned values, decomposing the error into a term due to projection and a term due to finite-memory approximation, which is related to the filter stability of the underlying system. For Q-learning with finite-memory variables, we consider discretization-based basis functions and provide convergence results and  error analysis for this setting.
\end{itemize}

\begin{remark}
Throughout the paper, $K < \infty$ denotes a generic constant. Its value may differ at different steps, but at each step it is uniform over other variables, such as time $t$ or random variables, within the given context.
\end{remark}

\section{Policy Evaluation for Non-Markov Processes}\label{pol_eval_sec}


\subsection{A Stochastic Approximation Result for Non-Markov Processes under Ergodicity}\label{stoc_app_sec}

We define the joint process $Z_t := (S_{t+1}, S_t, C_t, U_t)$.  We first present the assumptions for the main result.

\begin{assumption}\label{poisson}
\begin{itemize}
\item[i.] For any bounded function $f$, we have
\begin{align*}
\frac{1}{N} \sum_{t=1}^N f(Z_t) \to \int f(z) \pi(dz) \quad \text{a.s.}
\end{align*}
almost surely for some probability measure $\pi \in \P(\mathds{S}^2 \times \mathds{R} \times \mathds{U})$.

\item[ii.] For the matrix-valued function $A(Z_t)$ and the vector-valued function $b(Z_t)$, define
\begin{align}\label{gordin}
Y^A_t := \sum_{k=0}^\infty \| E[A(Z_{t+k})|\mathcal{F}_t] - A \|, \quad
Y^b_t := \sum_{k=0}^\infty \| E[b(Z_{t+k})|\mathcal{F}_t] - b \|,
\end{align}
where $A := \int A(z) \pi(dz)$ and $b := \int b(z) \pi(dz)$ and where we use the spectral norm for the matrices. We assume that these sequences are uniformly bounded in $L_2$: $\sup_t \|Y^A_t\|_2 < \infty$ and $\sup_t \|Y^b_t\|_2 < \infty$.

\item[iii.] $A(Z_t)$ and $b(Z_t)$ are uniformly bounded functions.
\end{itemize}
\end{assumption}

\begin{remark}
If $Z_t$ is strictly stationary, then $\|Y_t^A\|_2 = \|Y_0^A\|_2$ and $\|Y_t^b\|_2 = \|Y_0^b\|_2$ for all $t$. Without stationarity, the $L_2$ boundedness can still be extended to all $t$, as we show next.
\end{remark}

A sufficient condition for Assumption \ref{poisson} to hold without stationarity is via a summable strong mixing coefficient. For two sub-$\sigma$-algebras $\mathcal{A}, \mathcal{B} \subset \mathcal{F}$, define
\begin{align}\label{alpha}
\alpha(\mathcal{A},\mathcal{B}) := \sup_{A \in \mathcal{A}, B \in \mathcal{B}} | P(A \cap B) - P(A) P(B) |.
\end{align}

Let $\mathcal{F}_j^-$ denote the $\sigma-$algebra generated by $\{Z_t: t \leq j\}$. Similarly, $\mathcal{F}_j^+$ denote the  $\sigma-$algebra generated by $\{Z_t: t\geq j\}$. We recall the strong mixing coefficient of the process $\{Z_t\}$ defined by 
\begin{align}\label{mixing}
\alpha(k):=\sup_j \alpha(\mathcal{F}_j^- , \mathcal{F}^+_{j+k}).
\end{align}

\begin{assumption}\label{mixing_assmp}
\begin{itemize}
\item[i.] The random variables $Y_0^A$ and $Y_0^b$ defined in (\ref{gordin}) satisfy
\[
\|Y_0^A\|_2 < \infty \quad \text{and} \quad \|Y_0^b\|_2 < \infty.
\]
\item[ii.] The mixing coefficients $\alpha(k)$ defined in (\ref{mixing}) satisfy
\[
\sum_{k=0}^\infty \alpha(k)^{1/2} < \infty.
\]
\end{itemize}
\end{assumption}

\begin{lemma}
Assumption \ref{mixing_assmp} implies Assumption \ref{poisson} (ii). That is, if $\|Y_0^A\|_2 < \infty$, $\|Y_0^b\|_2 < \infty$, and $\sum_{k=0}^\infty \sqrt{\alpha(k)} < \infty$, then the sequences $\{Y^A_t\}$ and $\{Y^b_t\}$ are uniformly bounded for all $t$:
\[
\sup_t \|Y^A_t\|_2 < \infty, \quad \sup_t \|Y^b_t\|_2 < \infty.
\]
\end{lemma}

\begin{proof}
We proove the result for $Y_t^A$ only. We denote by
\begin{align*}
e_{t+k}:= \|E[A(Z_{t+k})|\mathcal{F}_t] -A \| 
\end{align*}
 We start with the following immediate bound:
\begin{align*}
\|Y_t^A\|_2 \leq \sum_{k=0}^\infty \|e_{t+k}\|_2.
\end{align*}
In what follows, we use the relation that for a $d\times d$ matrix $A$, we have
\begin{align*}
\|A\|\leq \|A\|_F\leq \sqrt{d}\|A\|
\end{align*}
where $\|A\|_F$ denotes the Frobenius norm.
For $e_{t+k}$, denoting by $\tilde{A}^{ij} (Z_{t+k})$ the $ij$-th entry of the matrix $A(Z_{t+k}) -A$ we can write
\begin{align*}
&\|e_{t+k}\|_2^2= E\left[ \| E[A(Z_{t+k})-A|\mathcal{F}_t] \|^2 \right]\leq E\left[  \sum_{i,j} E[\tilde{A}^{ij} (Z_{t+k})|\mathcal{F}_t ]^2 \right]\\
&= \sum_{i,j} E\left[ E[\tilde{A}^{ij} (Z_{t+k})|\mathcal{F}_t ]^2 \right]  = \sum_{i,j} E\left[ \tilde{A}^{ij} (Z_{t+k}) E[\tilde{A}^{ij} (Z_{t+k})|\mathcal{F}_t ]\right]\\
& =  \sum_{i,j} cov \left(  \tilde{A}^{ij} (Z_{t+k}), E[\tilde{A}^{ij} (Z_{t+k})|\mathcal{F}_t ] \right) +   \sum_{i,j} E\left[ \tilde{A}^{ij} (Z_{t+k})\right]^2\\
&=   \sum_{i,j} cov \left(  \tilde{A}^{ij} (Z_{t+k}), E[\tilde{A}^{ij} (Z_{t+k})|\mathcal{F}_t ] \right)+ \| E[\tilde{A}(Z_{t+k})]\|_F^2
\end{align*}
It is a standard result (see e.g. \cite{rio1993}) that for any bounded $f$
\begin{align*}
Cov(f(Z_{t+k}), E[f(Z_{t+k}|\mathcal{F}_t)]) \leq 4 \alpha(k) \|f\|^2_\infty.
\end{align*}
Using the boundedness of $A(z)$, we can then write for some $K<\infty$ that
\begin{align*}
&\|Y_t^A\|_2 \leq \sum_{k=0}^\infty \|e_{t+k}\|_2 \leq \sum_{k=0}^\infty \sqrt{K \alpha(k)} + \sum_{k=0}^\infty \| E[\tilde{A}(Z_{t+k})]\|_F \\
&\leq \sum_{k=0}^\infty \sqrt{K \alpha(k)} +  \sum_{k=0}^\infty \left\|E\left[ E[\tilde{A}(Z_{k})|\mathcal{F}_0] \right]\right\|_F \\
&\leq  \sum_{k=0}^\infty \sqrt{K \alpha(k)} + E\left[ \sum_{k=0}^\infty \left\| E[\tilde{A}(Z_k)|\mathcal{F}_0] \right\|_F  \right]\\
&\leq  \sum_{k=0}^\infty \sqrt{K \alpha(k)} + \sqrt{d}E[Y_0^A] \leq   \sum_{k=0}^\infty \sqrt{K \alpha(k)} +\sqrt{d} \|Y_0^A\|_2 <\infty.
\end{align*}
\end{proof}

The following proposition is a key result for the convergence of the policy evaluation algorithm under non-Markovian processes. The main technical tools, Lemmas \ref{as_conv_lemma} and \ref{L2_bound_lemma}, build primarily on \cite{benveniste2012adaptive}. 

In particular, the main challenge in the convergence proof arises from the error term embedded in the updates:
\[
\delta^\intercal \big[ A - A(Z_t) \big] \theta + \delta^\intercal \big[ b(Z_t) - b \big].
\]
In \cite{benveniste2012adaptive}, this term is analyzed for a \emph{Markov} process $Z_t$, where it is decomposed into a martingale difference term and summable telescoping terms using the Poisson equation satisfied by the Markov process under appropriate ergodicity conditions.  

For our key technical tools (Lemmas \ref{as_conv_lemma} and \ref{L2_bound_lemma}), we adopt a similar strategy. Namely, we show that the \emph{non-Markov} error term in our case also satisfies a Poisson equation under Assumption \ref{poisson}, we can then  decompose it into a martingale difference term and telescoping summable error terms. Although the overall approach follows similar steps as in \cite{benveniste2012adaptive}, the extension to non-Markov processes is not straightforward. The original analysis must be  revised  carefully, e.g. the verification of ergodicity conditions, control of the error terms, and the handling of conditional expectations. Therefore, even though the decomposition idea is similar, the non-Markov setting introduces significant technical challenges that require a tailored approach.

\begin{prop}\label{key_prop}
Suppose Assumption \ref{poisson} holds (or Assumption \ref{mixing} as a sufficient condition for Assumption \ref{poisson}) and that the stationary average matrix $A$ is positive definite. Consider the stochastic approximation iteration
\[
\theta_{t+1} = \theta_t + \alpha_t \big( -A(Z_t) \theta_t + b(Z_t) \big),
\]
where $A(Z_t)$ and $b(Z_t)$ are matrix and vector valued functions, respectively. Then, $\theta_t$ converges almost surely to a limit $\theta^*$ satisfying
$
A \theta^* = b,
$
where
\[
A = E[A(Z)] = \int A(z) \, \pi(dz), \quad 
b = E[b(Z)] = \int b(z) \, \pi(dz),
\]
and $\pi$ is the stationary distribution of the joint process $Z_t = \{S_{t+1}, S_t, C_t, U_t\}$.
\end{prop}

\begin{proof}
We start by adding and subtracting $A$ and $b$, and note that $b=A\theta^*$:
\begin{align*}
\theta_{t+1}=\theta_t + \alpha_t\left(-A(Z_t)\theta_t + A\theta_t    +b(Z_t)-b - A\theta_t + A\theta^* \right).
\end{align*}
Defining $\delta_t:=\theta_t-\theta^*$ and $M_t:= (A-A(Z_t))\theta_t + b(Z_t)-b$, and subtracting $\theta^*$ from each side, we get
\begin{align*}
\delta_{t+1}=\delta_t + \alpha_t\left(-A \delta_t + M_t\right).
\end{align*}
Taking the square of both sides, we write
\begin{align}\label{random_bnd}
\|\delta_{t+1}\|^2& =\|\delta_t\|^2 + 2\alpha_t\delta^\intercal_t\left[-A\delta_t+M_t\right]+\alpha^2_t \|-A\delta_t + M_t\|^2\nonumber\\
&\leq\|\delta_t\|^2 -2\alpha_t \sigma_{\min}\|\delta_t\|^2 + 2\alpha_t \delta^\intercal_t M_t + \alpha_t^2 2\sigma_{\max}\|\delta_t\|^2 + \alpha_t^2 2\|M_t\|^2
\end{align}
where $\sigma_{\min}$ and $\sigma_{\max}$ denote the minimum and the maximum eigenvalues of $A$, and where we used the bound that $(a+b)^2\leq 2a^2+2b^2$. Using the assumption that $A(Z_t),b(Z_t)$ are uniformly bounded, we can then have the following upper bound for $\|M_t\|$:
\begin{align*}
\|M_t\|&= \|(A-A(Z_t))\theta_t + b(Z_t)-b\|\\
& \leq  \|(A-A(Z_t))\| \|\theta_t\| + \|b(Z_t)-b\|\\
&\leq K (\|\delta_t\| + 1)
\end{align*}
for some $K<\infty$. We then also have that $\|M_t\|^2 \leq K(\|\delta_t\|^2+1)$ for some generic constant $K<\infty$. Using this, we get
\begin{align}\label{robbin_bound}
\|\delta_{t+1}\|^2 &\leq \|\delta_t\|^2 -2\alpha_t \sigma_{\min}\|\delta_t\|^2 + 2\alpha_t \delta^\intercal_t M_t + \alpha_t^2 K \|\delta_t\|^2 + \alpha_t^2 K\nonumber\\
&=(1+K\alpha_t^2)\|\delta_t\|^2 - 2\alpha_t \sigma_{\min}\|\delta_t\|^2 + 2\alpha_t \delta^\intercal_t M_t + \alpha_t^2 K.
\end{align}

We note that the  Robbins-Siegmund Lemma is not directly applicable since $2\alpha_t \delta^\intercal_t M_t $ is not guaranteed to be nonnegative. Nonetheless, we can show the convergence using alternative arguments. We first introduce the following stopping time:

 \begin{align}\label{stop_time}
 \sigma_n :=\inf\{t: \|\delta_t\|^2 > 2^n\}.
 \end{align}

\begin{lemma}\label{as_conv_lemma}
Under Assumption \ref{poisson}, we have that $\sum_{t=0}^k \mathds{1}_{\{t+1\leq \sigma_n\}}2\alpha_t  \delta^\intercal_t M_t$ converges almost surely. In particular 
$\sum_{t=0}^k2\alpha_t \delta^\intercal_t M_t$ converges almost surely on the event $\{\sigma_n=\infty\}$. 
\end{lemma}
\begin{proof}
The proof can be found in Appendix \ref{as_conv_lemma_proof}.
\end{proof}
\begin{lemma}\label{L2_bound_lemma}
We define the stoping time
 \begin{align*}
 \sigma(C) :=\inf\{t: \|\delta_t\|^2 > C\}.
 \end{align*}
Under Assumption \ref{poisson}, we have that for any $n<\infty$, 
\begin{align*}
E\left[\sup_{k>n} \mathds{1}_{\{k+1 \leq \sigma(C)\}}\left(\sum_{t=n}^{k} 2\alpha_t \delta_t^\intercal M_t\right)^2  \right] \leq K (1+C^2) \sum_{t=n}^{\infty} \alpha_t^2
\end{align*}
for some constant $K<\infty$.
\end{lemma}
\begin{proof}
The proof can be found in Appendix \ref{L2_bound_lemma_proof}.
\end{proof}

Multiplying, both sides by $\mathds{1}_{\{t+1\leq \sigma_n\}}$ in (\ref{robbin_bound}), and noting that $\mathds{1}_{\{t+2\leq \sigma_n\}} \leq \mathds{1}_{\{t+1\leq \sigma_n\}}$ and denoting by $\mathds{1}_{\{t+1\leq \sigma_n\}}\delta_t=:\hat{\delta}_t$
\begin{align}\label{stopped_ineq}
\|\hat{\delta}_{t+1}\|^2 &\leq(1+K\alpha_t^2)\|\hat{\delta}_t\|^2 - 2\alpha_t \sigma_{\min}\|\hat{\delta}_t\|^2 + 2\alpha_t \hat{\delta}^\intercal_t M_t + \alpha_t^2 K\nonumber\\
&\leq (1+K\alpha_t^2)\|\hat{\delta}_t\|^2 + 2\alpha_t \hat{\delta}^\intercal_t M_t + \alpha_t^2 K
\end{align}
Next, we define 
\begin{align*}
X_t:=\frac{\|\hat{\delta}_t\|^2}{\prod_{i=1}^{t-1}(1+K\alpha_i^2)}.
\end{align*}
We then observe that 
\begin{align*}
X_{t+1}\leq \frac{\|\hat{\delta}_t\|^2}{\prod_{i=1}^{t-1}(1+K\alpha_i^2)}+ \frac{2\alpha_t \hat{\delta}^\intercal_t M_t }{\prod_{i=1}^{t}(1+K\alpha_i^2)} + \frac{K \alpha_t^2}{\prod_{i=1}^{t}(1+K\alpha_i^2)}
\end{align*}
We now introduce the following notation:
\begin{align*}
&a_t: = \frac{2\alpha_t \hat{\delta}^\intercal_t M_t }{\prod_{i=1}^{t}(1+K\alpha_i^2)} +  \frac{K \alpha_t^2}{\prod_{i=1}^{t}(1+K\alpha_i^2)}
\end{align*}
which implies that 
$
X_{t+1}\leq X_t + a_t .
$
W define 
$
U_t:= X_t - \sum_{i=1}^{t-1}a_i.
$
With this notation, we write
\begin{align}\label{sup_mart}
E[U_{t+1}|\mathcal{F}_t] &= E[X_{t+1}|\mathcal{F}_t] - E[\sum_{i=1}^{t}a_i|\mathcal{F}_t] \nonumber\\
& =  E[X_{t+1}|\mathcal{F}_t] - \sum_{i=1}^{t}a_i\nonumber\\
& \leq  X_t + a_t  - \sum_{i=1}^{t}a_i\nonumber\\
& = X_t  - \sum_{i=1}^{t-1}a_i= U_t.
\end{align}
Using the proof of Lemma \ref{L2_bound_lemma}, we can show that 
\begin{align*}
E[(\sum_{i=1}^{t-1}a_i)^2] &\leq E\left[ (\sum_{i=1 }^{t-1}4\alpha_i \hat{\delta}^\intercal_i M_i)^2   \right] + \left(\sum_{i=1 }^{t-1}2 K\alpha_i^2\right)^2\\
&\leq K (1+2^{2n}) \sum_{i=1}^{\infty} \alpha_i^2 + \left(\sum_{i=1 }^{\infty}2 K\alpha_i^2\right)^2 < \infty.
\end{align*}
Furthermore, we have that 
\begin{align*}
E[X_t^2] \leq E[\|\hat{\delta}_t\|^2] \leq 2^{2n}.
\end{align*}
Combined, this implies that 
$
\sup_tE[U_t^2]<\infty.
$
Then, together with (\ref{sup_mart}), we can conclude that $U_t$ is a supermartingale with uniformly bounded $L_2$ norm, and thus $U_t$ converges almost surely. Furthermore, using the assumption on the learning rates and Lemma \ref{as_conv_lemma}, we also know that $\sum_{i=1}^{t-1} a_i$ converges almost surely. We then conclude that $X_t = U_t + \sum_{i=1}^{t-1} a_i$ converges almost surely. Since, $\prod_{i=1}^{t-1}(1+K\alpha_i^2)$ converges as well by assumptions on the learning rates, we have that 
$
\|\hat{\delta}_t\|^2
$ converges almost surely.

Going back to (\ref{stopped_ineq}), and rearranging the terms, we write
\begin{align*}
2\alpha_t \sigma_{\min}\|\hat{\delta}_t\|^2 \leq \|\hat{\delta}_t\|^2 -\|\hat{\delta}_{t+1}\|^2 + K\alpha_t^2\|\hat{\delta}_t\|^2 + 2\alpha_t \hat{\delta}^\intercal_t M_t + \alpha_t^2 K
\end{align*}
Noting that $\|\hat{\delta}_t\|^2\leq 2^{2n}$, and summing both sides, we get:
\begin{align*}
\sum_{t=0}^k 2\alpha_t \sigma_{\min}\|\hat{\delta}_t\|^2 &\leq \|\hat{\delta}_0\|^2 - \|\hat{\delta}_{k+1}\|^2+ \sum_{t=0}^k \left(K\alpha_t^2\|\hat{\delta}_t\|^2 + 2\alpha_t \hat{\delta}^\intercal_t M_t + \alpha_t^2 K\right)\\
&\leq \|\hat{\delta}_0\|^2 +  \sum_{t=0}^k \alpha_t^2 (1+ 2^{2n}) +  \sum_{t=0}^k  2\alpha_t \hat{\delta}^\intercal_t M_t.
\end{align*}
Using, Lemma \ref{as_conv_lemma} and the conditions on the learning rates, all the terms on the right hand side converges almost surely. Hence, we have that 
\begin{align*}
\sum_{t=0}^k 2\alpha_t \sigma_{\min}\|\hat{\delta}_t\|^2  <\infty
\end{align*}
almost surely, which implies that $\liminf_k \|\hat{\delta}_t\|^2 \to 0$. Since, we have proved earlier that $\|\hat{\delta}_t\|^2$ converges almost surely, the limit has to be 0, that is $\|\hat{\delta}_t\|^2\to 0$ almost surely. In particular, $\|{\delta}_t\|^2\to 0$ almost surely on the event $\{\sigma_n=\infty\}$.

Adapting the arguments of  \cite[Theorem 17]{benveniste2012adaptive} to the non-Markovian processes using Lemma \ref{L2_bound_lemma}, we can show that $P(\{\sigma_n=\infty\})\to 1$. We included the full proof of this in  Appendix \ref{bound_lemma_proof} for completeness.  
\begin{lemma}\label{bound_lemma}
Under Assumption \ref{poisson}, $P(\{\sigma_n=\infty\})\to 1$.
\end{lemma}
Lemma \ref{bound_lemma}, then concludes the proof. In particular, denoting by $A$ the event that $\|\delta_t\|\to 0$, and by $E_n$ the event that $\{\sigma_n=\infty\}$, we then have that $P(E_n \cap A^c)=0$ for all $n$. We can write
\begin{align*}
P((\cup_{n=1}^\infty E_n)\cap A^c)= P(\cup_{n=1}^\infty (E_n\cap A^c)) = \lim_n P(E_n \cap A^c)=0. 
\end{align*} 
Hence, together with the fact that $P(\cup_n E_n)=1$, we conclude that $P(A^c)=0$.

\end{proof}

\subsection{Stationary Regime MDP}\label{stat_MDP}

Recall joint process $Z_t:=\{S_{t+1},S_t,C_t,U_t\}$ where $S_t$ is a stochastic process representing the state process, $C_t$ is another process representing the cost realizations, and $U_t\sim\gamma(\cdot|S_k)$ is the control process generated by some policy $\gamma:\mathds{S}\to\P(\mathds{U})$.

Consider the invariant distribution $\pi$ of the process  $Z_t:=\{S_{t+1},S_t,C_t,U_t\}$ under Assumption \ref{poisson}.  We now define a Markov decision process for the stationary regime. 
The cost function and the transition kernel  are defined using the regular conditional distributions based on the stationary measure $\pi(\cdot)$ such that
\begin{align}\label{stat_model}
c(s,u)&:=E^\pi[C|s,u] \quad \forall s,u\in\mathds{S}\times\mathds{U}\nonumber\\
\eta(s_1\in A|s,u)& := E^\pi\left[\mathds{1}_{\{S_1\in A\}}|s,u\right]  \quad \forall s,u\in\mathds{S}\times\mathds{U}
\end{align}
where the expectation is with respect to the stationary distribution $\pi$ on $\{S_{t+1},S_t,C_t,U_t\}$. Note that the cost function and the transition model of this MDP depends on the stationary distribution and thus the policy $\gamma$ which leads to the particular stationary measure. We omit this dependence on the notation for brevity.

We define the following Bellman operator for  this stationary regime MDP under the policy $\gamma$, such that for $f\in L_2(\pi,\mathds{S})$, we write that
\begin{align}\label{appr_bell}
T^\gamma f(s) := \int_\mathds{U}\left(c(s,u) + \beta \int f(s_1)\eta(ds_1|s,u)\right)\gamma(du|s).
\end{align}
Similarly, for $g\in L_2(\pi,\mathds{S\times U})$, we write that
\begin{align}\label{appr_bell2}
T g(s,u) := c(s,u) + \beta \int {g^-}(s_1)\eta(ds_1|s,u)
\end{align}
where ${g^-}(s):=\inf_u g(s,u)$.

We define the value function of this MDP under the policy $\gamma$ by 
\begin{align*}
J^\pi_\beta(s_0,\gamma):= \sum_{t=0}^\infty \beta^t E^\gamma [c(\bar{S}_t,\bar{U}_t)|\bar{S}_0=s_0]
\end{align*}
where $\bar{S}_t$ denotes the {\it Markov} process with transition kernel $\eta(ds_1|s,u)$ defined in (\ref{stat_model}) and where $\bar{U}_t\sim\gamma(\cdot|\bar{S}_t)$. We put the bar notation to differentiate this from the original non-Markov process $S_t$.

\subsection{Linear Function Approximation and Projection}\label{lin_app_proj}
We consider the $L_2$ space of real valued functions on $s\in \mathds{S}$ with the measure $\pi\in \P(\mathds{S})$ under the usual inner product. The construction in this section is valid for any measure $\pi(\cdot)$, however, $\pi$ will mostly refer to the stationary measure of the process, and in particular its marginal on $S_t$.  

We introduce a set of basis functions $\{\phi^i(s)\}_{i=1}^d$ where  $\phi^i(s):\mathds{S}\to \mathds{R}$.  We denote by ${\bf \Phi}^\intercal:= [\phi^1,\dots,\phi^d]$ the vector of the basis functions.
\begin{assumption}
We assume for the rest of the paper that $
\|\phi^i\|_\infty\leq 1$
for all $i=1,\dots,d$.
\end{assumption}

\begin{assumption}\label{lin_ind}
We assume  for the rest of the paper that $\{\phi^i(s)\}$ are linearly independent in $L_2(\pi)$ such that $E\left[\Phi(S)\Phi^\intercal(S)\right]$ is invertible. 
\end{assumption}

We denote by $\Pi$ the projection map from $ L_2(\pi,\mathds{S})$ onto the span of ${\bf \Phi}^\intercal:= [\phi^1,\dots,\phi^d]$. In particular, for some $f\in  L_2(\pi,\mathds{S})$, $\Pi(f) = \theta_f^\intercal {\bf \Phi} $ where 
\begin{align}\label{proj}
\theta_f =\argmin_{\theta\in\mathds{R}^d} \sqrt{\int_{\mathds{S}} \left| f(s) - \theta^\intercal {\bf \Phi}(s)\right|^2 \pi(ds)}.
\end{align}

\begin{prop}\label{comp_contr}
The mapping $\Pi T^\gamma$ is a contraction under the $L_2$ norm, and thus admits a unique fixed point.
\end{prop}
\begin{proof}
For $f,g \in L_2(\pi,\mathds{S})$, we have that
\begin{align*}
\|\Pi T^\gamma (f) - \Pi T^\gamma (g)\|_2 \leq  \| T^\gamma (f) -  T^\gamma (g)\|_2 
\end{align*}
as the projection is non-expansive. Using the Jensen's inequality, we then have:
\begin{align*}
&\| T^\gamma (f) -  T^\gamma (g)\|_2\\
& \leq \beta \sqrt{   \int \left(f(s_1) - g(s_1)\right)^2\eta(ds_1|s,u)\gamma(du|s) \pi(ds) }\\
&=\beta  \sqrt{   \int \left(f(s_1) - g(s_1)\right)^2\pi(ds_1) } = \beta \|f-g\|_2.
\end{align*}
Above we used the fact that by construction  $\eta(ds_1|s,u)\gamma(du|s) \pi(ds) =\pi(ds_1,du,ds) $ since $\eta$ is the regular conditional distribution based on the stationary distribution on the joint process. Furthermore, the marginals of the stationary distribution on the consecutive state variables $S_t$ and $S_{t+1}$ coincide, which justifies the last step and thus the proof.
\end{proof}

\subsection{Convergence of the Policy Evaluation Algorithm}\label{pol_eval_conv_sec}
We consider the following algorithm
\begin{align}\label{main_iter}
\theta_{t+1}= \theta_t - \alpha_t {\bf \Phi}(S_t) \left[\theta_t^\intercal {\bf \Phi}(S_t) - C_t - \beta \theta^\intercal_t{\bf \Phi}(S_{t+1}) \right]
\end{align}
where $\alpha_t$ represents the learning rates, and where we use a single trajectory of $\{S_t,U_t,C_t\}_{t}$ under the policy $\gamma$.
\begin{theorem}\label{main_thm}
Under Assumption \ref{poisson} and \ref{lin_ind}, if the learning rates are such that $\sum_t\alpha_t = \infty$ and $\sum_t \alpha_t^2<\infty$, then the iterations in (\ref{main_iter}) converge to some $\theta^* \in\mathds{R}^d$. Denoting by $V(s):={\theta^*}^\intercal {\bf \Phi}(s)$, $V(s)$ is the fixed point of the joint mapping $\Pi T^\gamma $ where the mappings $\Pi$ and $T^\gamma$ are defined in (\ref{proj}) and (\ref{appr_bell}).
\end{theorem}

\begin{proof}
We use Proposition \ref{key_prop} with 
\begin{align*}
&A(S_t,S_{t+1})= -\beta{\bf \Phi}(S_t){\bf \Phi}^\intercal(S_{t+1})+ {\bf \Phi}(S_t){\bf \Phi}^\intercal(S_{t})\\
&b(S_t,C_t)={\bf \Phi}(S_t)C_t.
\end{align*}
The matrices $A$ and $b$ are defined under the invariant measure $\pi$ of the  joint process $(S_t,S_{t+1},U_t,C_t)$. 
 
 We need to show that the matrix $A$ is positive definite. 
 
 \begin{lemma}\label{helpful_lem}
\begin{align*}
(\theta-\theta^*)^\intercal E\left[ {\bf \Phi}(S)\left[ C+\beta \theta^\intercal{\bf \Phi}(S_1) - \theta^\intercal{\bf \Phi}(S)\right] \right]<0
\end{align*}
for any $\theta\neq\theta^*$ where $\theta^*$ corresponds to the fixed point of the operator $\Pi T^\gamma$, that is ${\theta^*}^\intercal{\bf \Phi}(s)$, where $\theta^*$ is unique under Assumption \ref{lin_ind}. 
\end{lemma}
\begin{proof}
Recall that $\Pi(C+\beta \theta^\intercal{\bf \Phi}(S_1) )$ denotes the projection map on the span of $\{\phi^i(s)\}$. Note that
\begin{align*}
\Pi(C+\beta \theta^\intercal{\bf \Phi}(S_1) ) &= \Pi\left(E\left[C+\beta \theta^\intercal{\bf \Phi}(S_1) |S\right]\right)\\
&=\Pi\left( T^\gamma ( \theta^\intercal{\bf \Phi}(S) ) \right).
\end{align*}
The first order conditions imply that $E\left[{\bf \Phi}(S)[ C+\beta \theta^\intercal{\bf \Phi}(S_1)  - \Pi(C+\beta \theta^\intercal{\bf \Phi}(S_1) ) ] \right]=0$. Then,  by adding and subtracting $ \Pi(C+\beta \theta^\intercal{\bf \Phi}(S_1) )$:
\begin{align*}
&(\theta-\theta^*)^\intercal E[ {\bf \Phi}(S)[ C+\beta \theta^\intercal{\bf \Phi}(S_1)  - \Pi(C+\beta \theta^\intercal{\bf \Phi}(S_1) )  \\
&\qquad\qquad\qquad +\left(  \Pi(C+\beta \theta^\intercal{\bf \Phi}(S_1) )-\theta^\intercal{\bf \Phi}(S)\right)] ]\\
& = (\theta-\theta^*)^\intercal E[ {\bf \Phi}(S)\left(  \Pi(C+\beta \theta^\intercal{\bf \Phi}(S_1) )-\theta^\intercal{\bf \Phi}(S) \right)] .
\end{align*}

In what follows, we use the equality $\Pi(C+\beta \theta^\intercal{\bf \Phi}(S_1) ) = \Pi\left( T^\gamma ( \theta^\intercal{\bf \Phi}(S) ) \right)$, and we add and subtract ${\theta^*}^\intercal{\bf \Phi}(S) =\Pi T^\gamma({\theta^*}^\intercal{\bf \Phi}(S)  )$ to use the contraction property of the composition operator $\Pi T^\gamma$ (see Proposition \ref{comp_contr}):
\begin{align*}
& (\theta-\theta^*)^\intercal E[ {\bf \Phi}(S)\left(  \Pi(C+\beta \theta^\intercal{\bf \Phi}(S_1) )-\theta^\intercal{\bf \Phi}(S) \right)] \\
&= (\theta-\theta^*)^\intercal E[ {\bf \Phi}(S) (\Pi T^\gamma(\theta^\intercal{\bf \Phi}(S))-  {\theta^*}^\intercal{\bf \Phi}(S)  ) ]\\
& +  (\theta-\theta^*)^\intercal E[ {\bf \Phi}(S) ( {\theta^*}^\intercal{\bf \Phi}(S) -  \theta^\intercal{\bf \Phi}(S)  )]\\
&\leq  \| (\theta-\theta^*)^\intercal  {\bf \Phi}(S)\|_2 \| \Pi T^\gamma(\theta^\intercal{\bf \Phi}(S))-  {\theta^*}^\intercal{\bf \Phi}(S)\|_2\\
& - \| (\theta-\theta^*)^\intercal  {\bf \Phi}(S)\|^2_2\\
& \leq  (\beta-1) \| (\theta-\theta^*)^\intercal  {\bf \Phi}(S)\|^2_2 <0
\end{align*}
where we used the Cauchy-Schwarz inequality, and the $L_2$ norm is with respect to the invariant measure $\pi$. The last step follows from the uniqueness of $\theta^*$.
\end{proof}
We then have that
 \begin{align*}
 &(-A) (\theta-\theta^*)=E\left[\beta {\bf \Phi}(S){\bf \Phi}^\intercal(S_{1})- {\bf \Phi}(S){\bf \Phi}^\intercal(S)  \right] (\theta-\theta^*)\\
 &= E\left[{\bf \Phi}(S)\left( C+\beta {\bf \Phi}^\intercal(S_{1})\theta- {\bf \Phi}^\intercal(S)\theta \right) \right]\\
 & \quad - E\left[{\bf \Phi}(S)\left( C+\beta {\bf \Phi}^\intercal(S_{1})\theta^*- {\bf \Phi}^\intercal(S)\theta^* \right) \right]\\
 &=E\left[{\bf \Phi}(S)\left( C+\beta {\bf \Phi}^\intercal(S_{1})\theta- {\bf \Phi}^\intercal(S)\theta \right) \right]
 \end{align*}
 where the last step follows from the fact that ${\bf \Phi}^\intercal(S)\theta^*$ is the fixed point of the operator $\Pi T^\gamma$ and that $\Pi(C+\beta \theta^\intercal{\bf \Phi}(S_1) ) = \Pi\left( T^\gamma ( \theta^\intercal{\bf \Phi}(S) ) \right)$.  Together with Lemma \ref{helpful_lem}, this shows that 
 \begin{align*}
 (\theta-\theta^*)(-A)(\theta-\theta^*)<0
 \end{align*}
 for all $\theta\neq\theta^*$, and thus using Proposition \ref{key_prop} we can conclude that $\theta_t$ converges to some $\theta'$ that satisfies $A\theta'=b$, which implies that
 \begin{align*}
 E\left[{\bf \Phi}(S)\left( C+\beta {\bf \Phi}^\intercal(S_{1})\theta'- {\bf \Phi}^\intercal(S)\theta' \right) \right]=0
 \end{align*}
 then as argued earlier, $\theta'$ also satisfies:
 \begin{align*}
 &E\bigg[{\bf \Phi}(S)\bigg( T^\gamma({\bf \Phi}^\intercal(S)\theta' )- {\bf \Phi}^\intercal(S)\theta' \bigg) \bigg]=0
 \end{align*}
 which in turn implies that ${\bf \Phi}^\intercal(S)\theta' $ is the fixed point of the operator $\Pi T^\gamma$. Since the fixed point is unique, we have that $\theta'=\theta^*$ which completes the proof.

\end{proof}

\subsection{Error Analysis for the Limit Value}\label{err_anal_sec}

Recall that \begin{align*}
J^\pi_\beta(s_0,\gamma)= \sum_{t=0}^\infty \beta^t E^\gamma [c(\bar{S}_t,\bar{U}_t)]
\end{align*}
denotes the value of the stationary regime MDP defined in Section \ref{stat_MDP}, and in particular it is the fixed point of the Bellman operator $T^\gamma$ given in (\ref{appr_bell}). We can then derive the following immediate bound:

\begin{prop}\label{l2_bound_prop}
Under the invariant measure $\pi$ of the joint process $(S_t,C_t,U_t)$ with the policy $\gamma$, we have that 
\begin{align*}
&\|J_\beta^\pi(S,\gamma) - {\theta^*}^\intercal {\bf \Phi}(S)\|_2 \leq\frac{1}{1-\beta}   \|J_\beta^\pi(S,\gamma) - \Pi (J_\beta^\pi(S,\gamma) ) \|_2.
\end{align*}
\end{prop}
\begin{proof}
We start with the following bound
\begin{align*}
\|J_\beta^\pi(S,\gamma) - {\theta^*}^\intercal {\bf \Phi}(S)\|_2&\leq \|J_\beta^\pi(S,\gamma) - \Pi T^\gamma(J_\beta^\pi(S,\gamma) ) \|_2 +  \|\Pi T^\gamma(J_\beta^\pi(S,\gamma) )-  {\theta^*}^\intercal {\bf \Phi}(S)\|_2\\
&\leq  \|J_\beta^\pi(S,\gamma) - \Pi (J_\beta^\pi(S,\gamma) ) \|_2 + \beta \|J_\beta^\pi(S,\gamma) - {\theta^*}^\intercal {\bf \Phi}(S)\|_2
\end{align*}
For the first term, since $J_\beta^\pi(s,\gamma)$ is the fixed point of the operator $T^\gamma$ (under the uniform norm), we have that $ \Pi T^\gamma(J_\beta^\pi(S,\gamma) ) = \Pi J_\beta^\pi(S,\gamma)$. For the second term, we use the fact that ${\theta^*}^\intercal {\bf \Phi}(S)$ is the fixed point of $\Pi T^\gamma$ which is a contraction under the $L_2$ norm. Combining the terms concludes the proof.
\end{proof}
The upper bound is related the projection error of the value function $J_\beta^\pi(s,\gamma)$ onto the span of ${\bf \Phi}$ under the $L_2$ norm of  the stationary measure $\pi$ with the policy $\gamma$. 
In the following, we derive an  upper bound on the uniform norm difference for near-linear value functions:
\begin{assumption}\label{near_lin}
We assume that there exists some $\hat{\theta}$ and some constant $\lambda<\infty$ such that 
\begin{align*}
\| J_\beta^\pi(s,\gamma) - \hat{\theta}^\intercal {\bf \Phi}(s)\|_\infty \leq \lambda. 
\end{align*}
\end{assumption}

\begin{prop}\label{uni_bound}
Under Assumption \ref{near_lin}, we have that 
\begin{align*}
&\|J_\beta^\pi(s,\gamma) - {\theta^*}^\intercal {\bf \Phi}(s)\|_\infty \leq \lambda \left(1+  \frac{2-\beta}{1-\beta}\sqrt{\frac{d}{\sigma_{\min}}}\right)
\end{align*}
where $\theta^*$ is the learned parameter with the iterations in (\ref{main_iter}). Furthermore,  $\sigma_{\min}$ is the minimum eigenvalue of the matrix $E[{\bf \Phi}(S) {\bf \Phi}^\intercal(S)  ] $ when $S$ is distributed with the invariant measure $\pi$.
\end{prop}
\begin{proof}
We begin by adding and subtracting $\hat{\theta}^\intercal {\bf \Phi}(s)$:
\begin{align*}
&\|J_\beta^\pi(s,\gamma) - {\theta^*}^\intercal {\bf \Phi}(s)\|_\infty\\
& \leq \| J_\beta^\pi(s,\gamma) - \hat{\theta}^\intercal {\bf \Phi}(s)\|_\infty+ \| \hat{\theta}^\intercal {\bf \Phi}(s) -  {\theta^*}^\intercal {\bf \Phi}(s)\|_\infty.
\end{align*}
The first term is bounded by $\lambda$ by assumption. We analyze the second term under the $L_2$ norm:
\begin{align*}
 &\| \hat{\theta}^\intercal {\bf \Phi}(S) -  {\theta^*}^\intercal {\bf \Phi}(S)\|_2 \\
 &\leq   \| \hat{\theta}^\intercal {\bf \Phi}(S) - J_\beta^\pi(S,\gamma)\|_2 + \|J_\beta^\pi(S,\gamma) - {\theta^*}^\intercal {\bf \Phi}(S)\|_2\\
 &\leq \lambda + \frac{1}{1-\beta}   \|J_\beta^\pi(S,\gamma) - \Pi (J_\beta^\pi(S,\gamma) ) \|_2\\
 &\leq \frac{2-\beta}{1-\beta}\lambda.
\end{align*}
For the second inequality, we used Proposition \ref{l2_bound_prop}. Furthermore, by Assumption \ref{near_lin}, the $L_2$ distance between $J_\beta^\pi(S,\gamma)$ and  $\hat{\theta}^\intercal {\bf \Phi}(S)$ is also bounded $\lambda$ as we work under probability measures. For the last inequality,  we use the fact that since $ \Pi (J_\beta^\pi(S,\gamma) ) $ is the projection of  $J_\beta^\pi(S,\gamma)$ under the $L_2$ norm of $\pi$, then it achieves the minimum $L_2$ distance to $J_\beta^\pi(S,\gamma)$, and thus it must achieve an error bound less than $\lambda$ that $\hat{\theta}$ achieves. 

On the other hand, we have that 
\begin{align*}
&\| \hat{\theta}^\intercal {\bf \Phi}(S) -  {\theta^*}^\intercal {\bf \Phi}(S)\|^2_2\\
&= (\theta^* - \hat{\theta}) E[{\bf \Phi}(S) {\bf \Phi}^\intercal(S)  ] (\theta^*-\hat{\theta}) \geq  \|\theta^* - \hat{\theta}\|^2_2 \sigma_{\min}
\end{align*}
where $\sigma_{\min}$ is the minimum eigenvalue of the matrix $E[{\bf \Phi}(S) {\bf \Phi}^\intercal(S)  ] $ when $S$ is distributed with the invariant measure $\pi$. Note that the $2$ norm for the $\theta$ vectors is the standard $2$ norm and not to be confused with the $L_2$ norm under $\pi$ over the functions. Combining what we have so far, we can write
\begin{align*}
 \|\theta^* - \hat{\theta}\|_2 \leq \frac{2-\beta}{1-\beta}\frac{\lambda}{\sqrt{\sigma_{\min}}}.
\end{align*}
Going back to the initial term, for any $S$, we have that 
\begin{align*}
&|J_\beta^\pi(s,\gamma) - {\theta^*}^\intercal {\bf \Phi}(s)|\\
& \leq | J_\beta^\pi(s,\gamma) - \hat{\theta}^\intercal {\bf \Phi}(s)|+ | \hat{\theta}^\intercal {\bf \Phi}(s) -  {\theta^*}^\intercal {\bf \Phi}(s)|\\
&\leq \lambda +  \|\theta^* - \hat{\theta}\|_2  \| {\bf \Phi}(s)\|_2\leq \lambda +  \frac{2-\beta}{1-\beta}\frac{\lambda\sqrt{d}}{\sqrt{\sigma_{\min}}}
\end{align*}
where we used the assumption that $\|\Phi^i\|_\infty\leq 1$ for all basis functions. Hence, the proof is complete.
\end{proof}

Proposition \ref{uni_bound} gives an error bound on the learned value and the value of the synthetic MDP constructed based on the stationary distribution of the original process. However, it does not answer the actual problem for which we are interested in the difference between the  value of the policy $\gamma$ under the true non-Markov dynamics of the state process $S_t$. This question requires a more careful analysis on the mixing properties of the process. In this paper, we will partially answer this question for partially observed MDPs under finite memory policies in Section \ref{pomdp_section} which is a special example of non-Markov processes.

\section{On Learning Approximately Optimal Q-Values }\label{q_learn_sec}
In this section, we shift our focus to approximately learning the optimal Q-values using linear function approximations. 
We extend our basis functions by using: $\{\phi^i(s,u)\}_{i=1}^d$ where  $\phi^i(s,u):\mathds{S}\times\mathds{U}\to \mathds{R}$.
We assume that $
\|\phi^i\|_\infty\leq 1$
for all $i=1,\dots,d$.

We denote the greedy policy by $\gamma_{\theta_t}(s)$ such that $\min_v \theta_t^\intercal {\bf \Phi}(s,v) =  \theta_t^\intercal {\bf \Phi}(s,\gamma_{\theta_t}(s))  $.
Consider the following iterations,
\begin{align}\label{opt_iter}
\theta_{t+1}= \theta_t - \alpha_t {\bf \Phi}(S_t,U_t) &\big[\theta_t^\intercal {\bf \Phi}(S_t,U_t)-  C_t - \beta \theta_t^\intercal {\bf \Phi}(S_{t+1},\gamma_{\theta_t}(S_{t+1})) \big]
\end{align}
where the actions are chosen under some time invariant exploration policy $\gamma:\mathds{S} \to\mathds{U}$.

The analysis of the optimal Q-learning iterations in (\ref{opt_iter}) differs from the one of policy evaluation given in (\ref{main_iter}). First note that the gain matrix is given by
\begin{align}\label{gain_mat_q}
&A(S_t,S_{t+1},U_t,\theta_t)= -\beta{\bf \Phi}(S_t,U_t){\bf \Phi}^\intercal(S_{t+1},\gamma_{\theta_t}(S_{t+1}))+ {\bf \Phi}(S_t,U_t){\bf \Phi}^\intercal(S_{t},U_t)
\end{align}
and thus the iterations are not fully linear in $\theta_t$. Nonetheless, the analysis in \cite{benveniste2012adaptive} holds for nonlinear functions under certain regularity conditions. Furthermore, this analysis can possibly be adapted to non-Markov processes as we have done in Section \ref{pol_eval_sec}. However, unlike the policy evaluation method (see Proposition \ref{comp_contr}), the joint projection-Bellman operator is not a contraction in general, mainly due to the discrepancy between the exploration policy and the greedy policy implicit in the Bellman operator. 

\begin{remark}
Note that another difference between the methods is due to the ergodicity assumptions. In particular,  the ergodicity condition of the policy evaluation methods in Assumption \ref{poisson} is stated for the gain matrix $A(Z_t)$ that is independent of $\theta_t$. For the Q-learning iterations, however, the gain matrix for the Q learning iterations (\ref{gain_mat_q}) depends on the parameter in a nonlinear way. Hence, one must adjust the ergodicity condition accordingly. In particular, we define
for any $f(Z_t)$ with $\|f\|_\infty\leq 1$, 
\begin{align*}
&\sum_{k=0}^\infty \|E[f(Z_{t+k})|\mathcal{F}_t] -\bar{f} \| =:Y^f_t 
\end{align*}
where $\bar{f}:=\int f(z)\pi(dz)$ with $\pi$ is the stationary distribution of the joint process $Z_t=(S_t,S_{t+1},C_t,U_t)$. The assumption is adapted such that $\sup_{\{\|f\|\leq 1\},t}\|Y_t^f\| < \infty$.
\end{remark}

Recall the Bellman operator defined for the stationary regime MDP in (\ref{appr_bell2})
\begin{align*}
T g(s,u) := c(s,u) + \beta \int \inf_v{g}(s_1,v)\eta(ds_1|s,u).
\end{align*}
Furthermore, $\Pi$ denotes the $L_2(\mathds{S}\times\mathds{U},\pi)$ orthogonal projection map on to the span of $\{\phi^i(s,u)\}_i$.  

The convergence of the iterations in (\ref{opt_iter}) is  related the convergence analysis of the deterministic sequence generated by the joint operator $\Pi T$. Unfortunately, this map is not a contraction outside of certain special cases:
\begin{itemize}
\item[1)] Clearly, one setting is where the cost function $c(s,u)$ and the transition model $\eta(ds_1|s,u)$ can be decomposed perfectly using the basis functions $\{\phi^i(s,u)\}$ (using real parameters for the cost function $c$, and signed measures for the kernel $\eta$). This setting is also known as linear MDPs, and the application of the Bellman operator does not push the iterations out of the linear span of the basis functions. Therefore, the joint map $\Pi T$ is equivalent to the application of the Bellman operator only, and the Bellman operator is a contraction under the uniform norm.
\item[2)] If the feature covariance induced by the greedy policy
is uniformly dominated by that induced by the exploration policy after discounting. 
\item[3)] When the basis functions are chosen using discretization of the space, then the projection maps the continuous space MDP to a discretized finite MDP, and thus the joint map preserves the uniform contraction property. 
\end{itemize}
In what follows, we explain the cases (2) and (3) in more detail.

\subsection{Greedy-Policy Covariance Dominance}\label{near_opt_exp}
One can show that the joint map $\Pi T$ is a contraction under the $L_2$ norm under a somewhat restrictive assumption on the auto-correlation matrices induced by the exploration policy and the greedy policy. This assumption is derived first by \cite{melo2008analysis} for Q-learning under linear functions approximation for Markov decision processes. For non-Markov processes, the same assumption is then needed for the stationary regime MDP that corresponds to the stationary distribution of the non-Markov process under the exploration policy.

We denote by
\begin{align}\label{sigma_gamma}
\Sigma_\gamma:=E\left[{\bf \Phi}(S,U){\bf \Phi}^\intercal(S,U)\right]
\end{align}
where $(S,U)$ is distributed according  to the invariant measure of the process $(S_t,U_t)$ under the exploration policy $\gamma$. We also denote by $\gamma_\theta(s)=\argmin_u \theta^\intercal {\bf \Phi}(s,u) $ the greedy policy for the parameter $\theta$. We define
\begin{align}\label{sigma_star}
\Sigma_\theta:=E\left[{\bf \Phi}(S,\gamma_\theta(S)){\bf \Phi}^\intercal(S,\gamma_\theta(S))\right]
\end{align}
where $S$ is distributed according to the invariant measure of $(S_t,U_t)$.

Recall the Bellman operator under the greedy action selection for the stationary regime MDP defined in (\ref{appr_bell2}) such that 
\begin{align*}
T g(s,u) := {c}(s,u) + \beta \int \inf_vf(s_1,v)\eta(ds_1|s,u)
\end{align*}
Recall also that $\Pi$, in this section, denotes the projection map over the span of the basis functions $\{\phi^i(s,u)\}_{i=1}^d$.

For the convergence of the algorithm, we impose the following assumption:
\begin{assumption}\label{rest_assmp}
For all $\theta\in\mathds{R}^d$
\begin{align*}
 \beta^2 \Sigma_\theta< \Sigma_\gamma.
\end{align*}
\end{assumption}
We note that this assumption is parallel to the assumption used in \cite{melo2008analysis}, and indicates that for large $\beta$, the greedy policy and the exploration policy are close to each other, which can be rather restrictive in practice.

\begin{prop}
Under Assumption \ref{rest_assmp}, the joint operator $\Pi T$ is a contraction in $L_2(\mathds{S}\times\mathds{U},\pi)$.
\end{prop}
\begin{proof}
The projection map is non-expansive, so we need to show that the Bellman map is a contraction in $L_2$. Let $f(s,u) = \theta_f^\intercal \Phi(s,u)$  and  $g(s,u) = \theta_g^\intercal \Phi(s,u)$.  We have that
\begin{align*}
&\| T (f) -  T (g)\|^2_2 \leq \beta^2 {   \int \left(\min_vf(s,v) - \min_vg(s,v)\right)^2 \pi(ds) }.
\end{align*}
We can show that 
$\left|\min_vf(s,v) - \min_vg(s,v)\right|\leq \max_\theta \left|f(s,\gamma_\theta(s))-g(s,\gamma_\theta(s))\right|$. Denoting the maximum achieving $\theta$ by $\bar{\theta}$:
\begin{align*}
&\beta^2 {   \int \left(\min_vf(s,v) - \min_vg(s,v)\right)^2 \pi(ds) } \\
&\leq\beta^2 (\theta_f-\theta_g)^\intercal \int {\bf \Phi}(s,\gamma_{\bar{\theta}}(s))  {\bf \Phi}^\intercal(s,\gamma_{\bar{\theta}}(s)) \pi(ds)(\theta_f-\theta_g)\\
&=  \beta^2(\theta_f-\theta_g)^\intercal \Sigma_{\bar{\theta}}  (\theta_f-\theta_g)\\
&<(\theta_f-\theta_g)^\intercal \Sigma_{\gamma}  (\theta_f-\theta_g) = \|f-g\|_2^2
\end{align*}
where we used Assumption \ref{rest_assmp} for the last inequality.
\end{proof}

\subsection{Convergence under Discretization}\label{disc_sec}
For the analysis so far, we have worked with the $L_2$ norm.  We have observed that the discrepancy between the exploration policy and the greedy policy within the Bellman operator makes the contraction analysis non-trivial for optimal Q-value estimation.

In this section, we discuss a special case for which the projection mapping does not expand the supremum norm of the functions. Accordingly, one can directly work with the uniform norm $\|\cdot\|_\infty$ for the contraction analysis.

Let $\{B^s_i\}_{i=1}^{M_s}$ be disjoint subsets of $\mathds{S}$ such that $\cup_{i=1}^{M_s}B^s_i=\mathds{S}$. Similarly, let $\{B^u_i\}_{i=1}^{M_u}$ be disjoint subsets of $\mathds{U}$ such that $\cup_{i=1}^{M_u}B^u_i=\mathds{U}$.   This discretization then implies a rectangular discretization on the joint state-action variables $(s,u)\in(\mathds{S\times U})$.
We denote by $\{A_i\}_{i=1}^{(M_s\times M_u)}$ for the resulting discretization bins of the joint $(s,u)\in(\mathds{S\times U})$ variable.
We define the following basis functions 
\begin{align*}
\phi^i(s,u)=\mathds{1}_{A_i}(s,u), \text{ for all } i=1,\dots, (M_s\times M_u)
\end{align*}
where $\mathds{1}_{A_i}(s,u)$ is the indicator function of the set $A_i$. Note that the projection map $\Pi$ is such that $\Pi(f)(s,u)=\theta^\intercal {\bf \Phi}(s,u)$, where $\theta=\Sigma^{-1}_\gamma E_{\pi}\left[ {\bf \Phi}(S,U)f(S,U)\right]$
for the invariant measure $\pi$ under the exploration policy $\gamma$ where $\Sigma_\gamma$ is defined in (\ref{sigma_gamma}). For the particular case of discretization, the basis functions $\phi^i$ are perfectly orthonormal and only one of them is equal to 1, and the rest are 0 for any input $(s,u)$. We then have that $
\Sigma^{-1}_\gamma(i,i)=\frac{1}{\pi(A_i)}
$
and it has $0$ entries for the non-diagonal elements. Thus, we can show that for some $(s,u)\in A_i$ 
\begin{align*}
\Pi(f)(s,u) &= \frac{\int_{A_i} f(s',u')\pi(ds',du')}{\pi(A_i)}\\
&=\int_{A_i} f(s',u')\pi_i(ds',du')\leq \sup_{s,u\in A_i} |f(s,u)|
\end{align*}
where $\pi_i(ds,du)$ is a probability measure normalized over $A_i$. Therefore, we have that $\|\Pi(f)\|_\infty\leq \|f\|_\infty$, and in particular,  the joint operator $\Pi T$ is a contraction under the \emph{supremum} norm.

 We denote by $\hat{\mathds{S}}:=\{s^1,\dots,s^{M_s}\}$ and $\hat{\mathds{U}}:=\{u^1,\dots,u^{M_u}\}$.
 Define a mapping $q_s:\mathds{S}\to\mathds{\hat{S}}$ and  $q_u:\mathds{U}\to\mathds{\hat{U}}$  such that 
$q_s(s)=s^i$  if  $s\in B^s_i$ and $q_u(u)=u^i \quad \text{ if } u\in B^u_i.$  

 In particular, the learning algorithm in (\ref{opt_iter}), takes the following particular form under discretization such that for any $s^i $ and $u^j$:
\begin{align*}
Q_{t+1}(s^i,u^j) = Q_t(s^i,u^j) - \alpha_t \mathds{1}_{\{\hat{S}_t=s^i, \hat{U}_t = u^j\}}\left[ Q_t(\hat{S}_t,\hat{U}_t) - C_t - \beta V_t(\hat{S}_{t+1}) \right]
\end{align*}
where $V_t(s):=\min_v Q_t(s,v)$ and where $\hat{S}_t:=q_s(S_t)$, $\hat{U}_t:=q_u(U_t)$.

Note that the above is a standard (tabular) Q-learning algorithm on the discretized state and action processes, $\hat{S}_t=q_s(S_t)$, $\hat{U}_t=q_u(U_t)$. The convergence of this algorithm under non-Markov processes is studied in \cite{kara2024qlearning} with random and state dependent learning rates:
\begin{theorem}
For all $s^i \in B_i^s$ and $u^j \in B_j^u$ and for $\hat{S}_t:=q_s(S_t)$, $\hat{U}_t:=q_u(U_t)$ consider
\begin{align*}
Q_{t+1}(s^i,u^j) = Q_t(s^i,u^j) - \alpha_t(s^i,u^j) \left[ Q_t(\hat{S}_t,\hat{U}_t) - C_t - \beta V_t(\hat{S}_{t+1}) \right].
\end{align*}
Assume that for any measurable bounded function $f$, we have that with probability one,
\begin{align*}
\frac{1}{N}\sum_{t=0}^{N-1} f(\hat{S}_{t+1},\hat{S}_t,\hat{U}_t,C_t)\to \int f(\hat{s}_1,\hat{s},\hat{u},c)\pi(d\hat{s}_1,d\hat{s},d\hat{u},dc)
\end{align*}
for some measure $\pi$ such that $\pi(\mathbb{\hat{S}}\times {s^i}\times {u^j}\times\mathbb{R})>0$ for any $(s^i,u^j)\in \mathbb{\hat{S}}\times \mathbb{\hat{U}}$. Furthermore, for the learning rates, we assume  $\alpha_t(s^i,u^j)=0$ unless $(\hat{S}_t,\hat{U}_t)=(s^i,u^j)$. Furthermore,
\begin{align*}
\alpha_t(s^i,u^j)=\frac{1}{1+\sum_{k=0}^t1_{\{\hat{S}_k=s^i,\hat{U}_k=u^j\}}}
\end{align*}
and with probability $1$. We then have that $Q_t(s^i,u^j)\to Q^*(s^i,u^j)$ almost surely for each $(s^i,u^j)\in\mathbb{\hat{S}}\times\mathbb{\hat{U}}$ pair where $Q^*$ is the optimal Q-values for the stationary regime MDP constructed in Section \ref{stat_MDP} for the discretized state and actions.
\end{theorem}

\section{Function Approximation for POMDPs using Finite Memory}\label{pomdp_section}

\subsection{Partially Observed Markov Decision Processes}
Let $\mathds{X} \subset \mathds{R}^m$ denote a Borel set which is the state space of a POMDP for some $m\in\mathds{N}$. Let
$\mathds{Y} \subset \mathds{R}^n$ be another Borel set denoting the observation space of the model, and let the state be observed through an
observation channel $O$. The observation channel, $O$, is defined as a stochastic kernel (regular
conditional probability) from  $\mathds{X}$ to $\mathds{Y}$, such that
$O(\,\cdot\,|x)$ is a probability measure on the sigma algebra $\mathcal{B}(\mathds{Y})$ of $\mathds{Y}$ for every $x
\in \mathds{X}$, and $O(A|\,\cdot\,): \mathds{X}\to [0,1]$ is a Borel
measurable function for every $A \in \mathcal{B}(\mathds{Y})$.
$\mathds{U}\in \mathds{R}^l$ denotes the action space. An {\em admissible policy} $\gamma$ is a
sequence of control functions $\{\gamma_t,\, t\in \Zplus\}$ such
that $\gamma_t$ is measurable with respect to the $\sigma$-algebra
generated by the information variables
$
I_t=\{Y_{[0,t]},U_{[0,t-1]}\}, \quad t \in \mathds{N}, \quad
  \quad I_0=\{Y_0\},
$
where $U_t=\gamma_t(I_t),\quad t\in \Zplus$,
are the $\mathds{U}$-valued control
actions and 
$Y_{[0,t]} = \{Y_s,\, 0 \leq s \leq t \}, \quad U_{[0,t-1]} =
  \{U_s, \, 0 \leq s \leq t-1 \}.$ We define $\Gamma$ to be the set of all such admissible policies. The update rules of the system are determined by
relationships:
\[  \Pr\bigl( (X_0,Y_0)\in B \bigr) =  \int_B \mu(dx_0)O(dy_0|x_0), \quad B\in \mathcal{B}(\mathds{X}\times\mathds{Y}), \]
where $\mu$ is the (prior) distribution of the initial state $X_0$, and
\begin{eqnarray*}
\label{eq_evol}
 &\Pr\biggl( (X_t,Y_t)\in B \, \bigg|\, (X,Y,U)_{[0,t-1]}=(x,y,u)_{[0,t-1]} \biggr)\\
& = \int_B \mathcal{T}(dx_t|x_{t-1}, u_{t-1})O(dy_t|x_t),  
\end{eqnarray*}
$B\in \mathcal{B}(\mathds{X}\times\mathds{Y}), t\in \mathds{N},$ where $\mathcal{T}$ is the transition kernel of the model which is a stochastic kernel from $\mathds{X}\times
\mathds{U}$ to $\mathds{X}$.  We let the objective of the agent (decision maker) be the minimization of the infinite horizon discounted cost, 
  \begin{align}\label{criterion1}
    J_{\beta}(\mu,\gamma)= E_\mu^{\gamma}\left[\sum_{t=0}^{\infty} \beta^t c(X_t,U_t)\right]
  \end{align}
 \noindent for some discount factor $\beta \in (0,1)$, over the set of admissible policies $\gamma\in\Gamma$, where $c:\mathds{X}\times\mathds{U}\to\R$ is a Borel-measurable stage-wise cost function and $E_\mu^{\gamma}$ denotes the expectation with initial state probability measure $\mu$ and transition kernel $\mathcal{T}$ and the channel $O$ under policy $\gamma$. Note that $\mu\in\mathcal{P}(\mathds{X})$, where we let $\mathcal{P}(\mathds{X})$ denote the set of probability measures on $\mathds{X}$. We define the optimal cost for the discounted infinite horizon setup as a function of the priors as
\begin{align}\label{opt_val}
  J_{\beta}^*(\mu)&=\inf_{\gamma\in\Gamma} J_{\beta}(\mu,\gamma).
\end{align}
For the analysis of partially observed MDPs, a common approach is to reformulate the problem as a fully observed MDP where the decision maker keeps track of the posterior distribution of the state $X_t$ given the available history $I_t$, also called the belief MDP. In what follows, we will use an alternative yet related reformulation based on finite-memory (window) information variables.

\subsection{Reduction to Fully Observed Using Finite-Memory Variables}\label{finite_memory}
 The following construction  is mostly taken from \cite{kara2023convergence}, however, we present the method in detail for completeness.

 We construct a fully observed MDP reduction using the predictor from $N$ stages earlier and the most recent $N$ information variables (that is, measurements and actions). 
Consider the following state variable at time $t$:
\begin{align}\label{finite_belief_state}
{z}_t=(\mu_{t-N},h_t)
\end{align}
where, for $N\geq 1$
\begin{align*}
\mu_{t-N}&=Pr(X_{t-N}\in \cdot|y_{t-N-1},\dots,y_0,u_{t-N-1},\dots,u_0),\\
h_t&=\{y_t,\dots,y_{t-N},u_{t-1},\dots,u_{t-N}\}
\end{align*}
and $h_t=y_t$ for $N=0$
with $\mu$ being the prior probability measure on $X_0$. Note that although, the finite-memory variable $h_t$ depends on the memory length $N$, we drop this dependence for notational convenience.

The state space with this representation is ${\mathcal{Z}}=\P(\mathds{X})\times \mathds{Y}^{N+1}\times \mathds{U}^{N}$ where we equip ${\mathcal{Z}}$ with the product topology where we consider the weak convergence topology on the $\P(\mathds{X})$ and the usual (coordinate) topologies on $\mathds{Y}^{N+1}\times \mathds{U}^{N}$. 

We can now define the stage-wise cost function and the transition probabilities. Consider the new cost function $\hat{c}:{\mathcal{Z}}\times \mathds{U}\to \mathds{R}$,  
\begin{align}\label{hat_cost}
&\hat{c}({z}_t,u_t)=\hat{c}(\mu_{t-N},h_t,u_t)=\int_\mathds{X}c(x_t,u_t)P^{\mu_{t-N}}(dx_t|y_t,\dots,y_{t-N},u_{t-1},\dots,u_{t-N}).\nonumber
\end{align}
Furthermore, we can define the transition probabilities for $N=1$ (for simplicity) as follows: for some $A\in \B(\mathcal{Z})$ such that 
\[\ A= B\times\{\hat{y}_{t-N+1},\hat{u}_{t},\dots,\hat{u}_{t-N+1}\},\quad  B\in\B(\P(\mathds{X})) \]
we write
\begin{align*}
&Pr({z}_{t+1}\in A|{z}_{t},\dots,{z}_0,u_{t},\dots,u_0) \\
&=Pr(\mu_t\in B,\hat{y}_{t+1},\hat{y}_{t},\hat{u}_{t}|\mu_{[t-1,0]},y_{[t,0]},u_{[t,0]})\\
&=\mathds{1}_{\{y_{t},u_{t}=\hat{y}_{t},\hat{u}_{t},G(\mu_{t-1},y_{t-1},u_{t-1})\in B\}} P^{\mu_{t-1}}(\hat{y}_{t+1}|y_t,y_{t-1},u_{t},u_{t-1})\\
&=Pr(\mu_{t}\in B,\hat{y}_{t+1},\hat{y}_{t},\hat{u}_{t}|\mu_{t-1},y_t,y_{t-1},u_t,u_{t-1})\\
&=Pr({z}_{t+1}\in A|{z}_t,u_t)=:\int_A{\eta}(d{z}_{t+1}|{z}_t,u_t)
\end{align*}
where the map $G$ is defined as 
\begin{align*}
&G(\mu_{t-1},y_{t-1},u_{t-1})=P^\mu(X_{t}\in\cdot|y_{t-1},\dots,y_0,u_{t-1},\dots,u_0).
\end{align*}
For some admissible policy $\gamma$, and some initial state $z_0\in\mathcal{Z}$ we write its induced cost as 
\begin{align*}
{J}_\beta(z_0,\gamma) = \sum_{t=0}^\infty \beta^t E^\gamma[\hat{c}(Z_t,U_t)].
\end{align*}
Respectively, we denote the optimal value function by ${J}^*_\beta(z_0)$. Note that this construction is without loss of optimality. In particular, for a fixed $\mu_{-N}$, assuming some arbitrary policy $\gamma$ acts from time $-N$ through $-1$, one can then show that 
\begin{align*}
E\left[ {J}_\beta^*(Z_0)\right] = E\left[ {J}_\beta^*(\mu_{-N} , H_0)\right] = E[J_\beta^*(\mu_0)]
\end{align*}
where the expectation on the left is with respect to $H_0=\{Y_0,\dots,Y_{-N},U_{-1},\dots, U_{-N} \}$, and on the right with respect to $\mu_{0}=Pr(X_{0}\in \cdot|Y_{-1},\dots,Y_{-N},U_{-1},\dots,U_{-N})$. Note that $J_\beta^*(\mu_0)$ is the optimal value function defined in (\ref{opt_val}).

Hence, we have a fully observed MDP, with the cost function $\hat{c}$, transition kernel ${\eta}$ and the state space ${\mathcal{Z}}$.

\subsection{Approximation of the Finite-Memory Belief-MDP}\label{app_section}
The finite-memory belief MDP model constructed in the previous section lives in the state space
\begin{align*}
{\mathcal{Z}}=\bigg\{&\pi,y_{[0,N]},u_{[0,N-1]}:\pi\in\mathcal{P}(\mathds{X}), y_{[0,N]}\in{\mathds{Y}}^{N+1}, u_{[0,N-1]}\in\mathds{U}^N\bigg\},
\end{align*}
where the first coordinate summarizes the past information, and the second and the last coordinates carry the information from the most recent $N$ time steps.

Consider the following set ${\mathcal{Z}}_{\pi}$ for a fixed $\pi\in\P(\mathds{X})$
\begin{align*}
{\mathcal{Z}}_{\pi}=\bigg\{\pi,y_{[0,N]},u_{[0,N-1]}: y_{[0,N]}\in{\mathds{Y}}^{N+1}, u_{[0,N-1]}\in\mathds{U}^N\bigg\}
\end{align*}
such that the state at time $t$ is $\hat{z}_t=(\pi,h_t)$. Compared to the state ${z}_t=(\mu_{t-N},h_t)$ defined in (\ref{finite_belief_state}), this approximate model uses $\pi$ as the predictor, no matter what the real predictor at time $t-N$ is.

Since $\pi$ is fixed, we can consider the state to be only $h_t$.
The cost function is defined as 
\begin{align}\label{cost_window}
&\hat{c}_\pi({h}_t,u_t)=\hat{c}(\pi,h_t,u_t)=\int_\mathds{X}c(x_t,u_t)P^{\pi}(dx_t|y_t,\dots,y_{t-N},u_{t-1},\dots,u_{t-N}).
\end{align}
We define the controlled transition model by 
\begin{align}\label{eta_N}
&{\eta}_\pi({h}_{t+1}|{h}_t,u_t):={\eta}\bigg(\P(\mathds{X}),h_{t+1}|\pi,h_t,u_t\bigg).
\end{align}

For simplicity, if we assume $N=1$, then the transitions can be rewritten for some $h_{t+1}=(\hat{y}_{t+1},\hat{y}_t,\hat{u}_t)$ and $h_t=(y_t,y_{t-1},u_{t-1})$
\begin{align}\label{alt_eta_N}
&{\eta}_\pi(\hat{y}_{t+1},\hat{y}_t,\hat{u}_t|y_t,y_{t-1},u_{t-1},u_t)={\eta}(\P(\mathds{X}),\hat{y}_{t+1},\hat{y}_t,\hat{u}_t|\pi,y_t,y_{t-1},u_{t-1},u_t)\nonumber\\
&=\mathds{1}_{\{y_t=\hat{y}_t,u_t=\hat{u}_t\}}P^{\pi}(\hat{y}_{t+1}|y_t,y_{t-1},u_t,u_{t-1}).
\end{align} 

We define the following Bellman operator under a finite-memory policy $\gamma^N$ for this model such that for any $f$
\begin{align}\label{fin_mem_Bellman}
T^N f(h)= \hat{c}_\pi(h^N,\gamma^N(h))+ \beta \int f(h_1)\eta_\pi(dh_{t+1}|h,\gamma^N(h))
\end{align}



We denote the optimal value function for the approximate model by $J_\beta^N$. 
Note that $J^N_\beta$ is defined on the set ${\mathcal{Z}}_{\pi}$. However, we can simply extend it to the set ${\mathcal{Z}}$ by defining it as constant over $\P(\mathds{X})$ for the first coordinate.  

We also note that since the predictor $\pi$ is fixed, $J_\beta^N$ can be thought as a function on $h_t$, the finite-memory information variables. 

We define the following constant:
\begin{align}\label{loss_constant}
L_t:=\sup_{\hat{\gamma}\in\hat{\Gamma}}E_{\mu_0}^{\hat{\gamma}}\bigg[\|P^{\mu_t}(X_{t+N}\in\cdot|Y_{[t,t+N]},U_{[t,t+N-1]})
-P^{\pi}(X_{t+N}\in\cdot|Y_{[t,t+N]},U_{[t,t+N-1]})\|_{TV}\bigg]
\end{align} 
which is the expected value on the total variation distance between the posterior distributions of $X_{t+N}$ conditioned on the same observation and control action variables $Y_{[t,t+N]},U_{[t,t+N-1]}$ when the prior distributions of $X_{t}$ are given by $\mu_t$ and $\pi$. This filter stability term plays a significant role in the error analysis that follows. One can show that $L_t \to 0$ as $N\to 0$ (in some cases, exponentially fast) under certain assumptions. We refer the reader to \cite{kara2020near,kara2023convergence,mcdonald2020exponential} for further details on this analysis. 

\begin{prop}\cite[Theorem 3.3]{kara2023convergence}\label{cont_bound}
For ${z}_0=(\mu_0,h_0)$, with a policy $\hat{\gamma}$ acting on the first $N$ steps, we have that
\begin{itemize}
\item For a finite-memory policy (not necessarily optimal) $\gamma^N$
\begin{align*}
E_{\mu_0}^{\hat{\gamma}}\left[\left|J^N_\beta(h_0,\gamma^N) -J_\beta({z}_0,\gamma^N)\right|\right]\leq  \frac{\|c\|_\infty }{(1-\beta)}\sum_{t=0}^\infty\beta^tL_t
\end{align*}
\item For the difference between the value functions we have
\begin{align*}
&E_{\mu_0}^{\hat{\gamma}}\left[\left|J^N_\beta(h_0) -J^*_\beta({z}_0)\right|\right]\leq  \frac{\|c\|_\infty }{(1-\beta)}\sum_{t=0}^\infty\beta^tL_t
\end{align*}
where the expectation is with respect to the random realizations of the initial finite-memory variables $h_0$.
\end{itemize}
\end{prop}

\subsection{Finite-Memory Policy Evaluation for POMDPs}
In this section, we aim to learn an approximate value for a given finite-memory policy. In particular, we use the methods in Section \ref{pol_eval_sec}, by setting 
\begin{align*}
s_t = h_t=\{y_t,\dots,y_{t-N},u_{t-1},\dots,u_{t-N}\}.
\end{align*}
In particular, we also have that $\mathds{S}=\mathds{Y}^{N-1}\times\mathds{U}^N$.  
We use the same iterations in (\ref{main_iter}) such that
\begin{align}\label{fin_mem_pol}
\theta_{t+1}= \theta_t - \alpha_t {\bf \Phi}(S_t) \left[\theta_t^\intercal {\bf \Phi}(S_t) - C_t - \beta \theta^\intercal_t{\bf \Phi}(S_{t+1}) \right]
\end{align}
for given basis functions $\{\phi^i\}_i$ defined on  $ \mathds{S}=\mathds{Y}^{N-1}\times\mathds{U}^N$.

\begin{corollary}[to Theorem \ref{main_thm}]
 Let Assumption \ref{lin_ind} and  Assumption \ref{poisson} hold for $Z_t= (S_t,S_{t+1},c(X_t,U_t),U_t)$ where $S_t$ is the finite-memory variable under the finite-memory policy $\gamma^N$. Then, the iterations in (\ref{fin_mem_pol}) converge to some $\theta^*$. 
\end{corollary}

\noindent{\bf Ergodicity} In this part, we study the long run behavior of the finite-memory process $\{h_t\}$. We note that this process is not a Markov chain. However, the joint process $(h_t,x_{t},u_t)$ is a Markov chain under a finite-memory policy $\gamma$. For example, for $N=2$ and for some $B_1, B_2 \in  \mathcal{B}(\mathds{Y}), B_3,B_4 \in \mathcal{B}(\mathds{U}), B_5 \in \mathcal{B}(\mathds{X}) $, denoting by $I_{t+1} = \{(y,x,u)_ {t+1},\dots,(y,x,u)_{0}\}$
\begin{align*}
&Pr({Y}_{t+2} \in B_1,{Y}_{t+1} \in B_2, U_{t+1} \in B_3, X_{t+2}\in B_5,  {U}_{t+2}\in B_4 | I_{t+1} )\\
&= \int_{x_{t+1}\in B_5}\int_{x_{t+2}\in\mathds{X}} \int_{y_{t+2}\in B_2} \int_{u_{t+2}\in B_4}\mathds{1}_{(y_{t+1} \in B_2, u_{t+1}\in B_3)} \\
&  \qquad\gamma(du_{t+2}|y_{t+2},y_{t+1},u_{t+1})O(dy_{t+2}|x_{t+2}) \mathcal{T}(dx_{t+2}|x_{t+1},u_{t+1})
\end{align*}
which shows that the joint process is a Markov chain. We note that the geometric ergodicity of this Markov process is a sufficient condition for Assumption \ref{poisson} under the finite-memory policy $\gamma^N$.

However, it is not possible to guarantee this condition solely using the properties of the transition kernel $\mathcal{T}(\cdot|x,u)$  in general. This is due to the fact that the finite-memory variable $h_t$ contains the past control actions, and thus the dependence of the control policies on the past control actions makes the ergodicity analysis non-trivial.  For example, for a policy of type $u_t\sim\gamma(\cdot|u_{t-1})$, the ergodicity of the action process and thus the finite-memory process, clearly depends on the randomized policy $\gamma(\cdot|u_{t-1})$.

We  note that if the finite-memory policy $\gamma$ and the transition kernel $\mathcal{T}$ satisfy a minorization condition, then the augmented process is exponentially ergodic and thus satisfies Assumption \ref{poisson}.
\begin{assumption}\label{minor}
There exist non-trivial measures $\lambda_x(\cdot)$ and  $\lambda_u(\cdot)$ such that 
\begin{align*}
\mathcal{T}(dx_1|x,u)&\geq\lambda_x(dx_1)\\
\gamma(du|h)&\geq \lambda_u(du)
\end{align*}
for all $(x,u)\in\mathds{X}\times\mathds{U}$ and for all $h\in \mathds{Y}^{N}\times \mathds{U^{N-1}}$.
\end{assumption}
\begin{lemma}
Assumption \ref{minor} implies Assumption \ref{poisson} for the joint process $(h_t,x_t,u_t)$. In particular, under Assumption \ref{minor}, the augmented Markov chain $(h_t,x_{t},u_t)$ is exponentially ergodic under the finite-memory policy.
\end{lemma}
\begin{proof}
We give a proof for $N=2$: consider the two step transition for the chain $(h_t,x_{t},u_t)$ for some starting point $(y_1,y_0,u_0,x_1,u_1)$:
\begin{align*}
&Pr(dy_3, dy_2,du_2,dx_3,du_3|y_1,y_0,u_0,x_1,u_1)\\
&=\int_{x_2\in\mathds{X}} \gamma(du_3|y_3,y_2,u_2)O(dy_3|x_3)\mathcal{T}(dx_3|x_2,u_2)\gamma(du_2|y_2,y_1,u_1) O(dy_2|x_2) \mathcal{T}(dx_2|x_1,u_1)\\
&\geq \int_{x_2} \gamma(du_3|y_3,y_2,u_2)O(dy_3|x_3)\mathcal{T}(dx_3|x_2,u_2)\lambda_u(du_2)O(dy_2|x_2) \lambda_x(dx_2)\\
&=:\lambda_h(du_3,dy_3, dy_2,du_2,dx_3)
\end{align*}
the non-trivial measure $\lambda_h(\cdot)$  is independent of the starting point, and thus it can be shown that $(h_t,x_{t},u_t)$ is exponentially ergodic (see e.g.  \cite[Lemma 3.3]{hernandez2012adaptive}.
\end{proof}
\begin{remark}
For any finite-memory policy $\gamma$ that does not satisfy Assumption \ref{minor}, one can always construct a perturbed version that does satisfy this assumption. In particular, let $\gamma'$ be an arbitrary policy that satisfies the minorizarion policy. Then, the perturbed policy $\hat{\gamma}(du|h)=(1-\epsilon)\gamma(du|h)+\epsilon \gamma'(du|h)$ satisfies Assumption \ref{minor} by construction.
\end{remark}

\noindent{\bf Error bounds for the learned value}
In the previous section, we observed that  using the iterations (\ref{main_iter}), one can learn the fixed point of the operator $\Pi T^N$ where $\Pi$ is the projection map and $T^N$ is defined in (\ref{fin_mem_Bellman}). In the following, we compare the learned value function ${\theta^*}^\intercal {\bf \Phi}(h)$ with the fixed point of the operator $T^N$. We note that the fixed point of the operator $T^N$ is the value function of the finite-memory policy $\gamma^N$ for the approximate model constructed in Section \ref{app_section} which we denote by $J_\beta^N(h,\gamma^N)$. However, this is not the value of the finite-memory policy in the original partially observed environment.

The next result  provides an error upper-bound for the learned value function with respect to the true value of the finite-memory policy in the original environment.
\begin{assumption}\label{near_lin_pomdp}
We assume that there exists some $\hat{\theta}$ and some constant $\lambda<\infty$ such that 
\begin{align*}
\| J_\beta^N(h,\gamma^N) - \hat{\theta}^\intercal {\bf \Phi}(h)\|_\infty \leq \lambda. 
\end{align*}
\end{assumption}
\begin{theorem}
Assume Assumption \ref{near_lin_pomdp} holds. We assume that the unobserved state initiates at time $-N$ according to some $\mu_{-N}\in \P(\mathds{X})$, and the finite-memory policy $\gamma^N$ starts acting at time $t=0$. We denote by $h_0$, the finite-memory variables from time $t=-N$ to $t=0$. For ${z}_0=(\mu_{-N},h_0)$, with a policy $\hat{\gamma}$ acting on the first $N$ steps, we have that
\begin{align*}
&E_{\mu_{-N}}^{\hat{\gamma}}\left[\left|J_\beta({z}_0,{\gamma^N}) -{\theta^*}^\intercal {\bf \Phi}(h_0)\right|\right]\\
&\leq  \frac{\|c\|_\infty }{(1-\beta)}\sum_{t=0}^\infty\beta^tL_t + \lambda\left(1+ \frac{2-\beta}{1-\beta}  \sqrt{\frac{d}{\sigma_{\min}}  }\right)
\end{align*}
 where the expectation is with respect to the random realizations of the initial finite-memory variables $h_0$.  Furthermore,  $\sigma_{\min}$ is the minimum eigenvalue of the matrix $E[{\bf \Phi}(H) {\bf \Phi}^\intercal(H)  ] $ when $H$ is distributed with the invariant measure $\pi$.
\end{theorem}
\begin{proof}
The proof is an application of Proposition \ref{cont_bound} and Proposition \ref{uni_bound}.
\begin{align*}
&E_{\mu_{-N}}^{\hat{\gamma}}\left[\left|J_\beta({z}_0,{\gamma^N}) -{\theta^*}^\intercal {\bf \Phi}(h_0)\right|\right]\\
&\qquad\qquad\leq  E_{\mu_{-N}}^{\hat{\gamma}}\left[\left|J_\beta({z}_0,{\gamma^N}) -  J_\beta^N(h,\gamma^N) \right|\right] + E_{\mu_{-N}}^{\hat{\gamma}}\left[\left| J_\beta^N(h,\gamma^N) - {\theta^*}^\intercal {\bf \Phi}(h_0)  \right|\right]
\end{align*}
the first term is bounded by Proposition \ref{cont_bound} and the second term is bounded by Proposition \ref{uni_bound}.
\end{proof}

\subsection{Convergence and Neal Optimality under Discretization for POMDPs}
As explained in Section \ref{q_learn_sec} convergence of the Q-learning algorithm is usually not guaranteed expect for a few special cases. As also explained in Section  \ref{disc_sec}, discretization based basis functions is one of these special cases. 

We provide a discretization method for the finite-memory variables for POMDPs in this section, and present the convergence and near optimality of the resulting algorithm building on  \cite{devran2025near}.

 For a weak Feller belief MDP (\cite{FeKaZa12, KSYWeakFellerSysCont}), \cite[Theorem 3.16]{SaLiYuSpringer}   established near optimality of finite action policies. If $\mathds{U}$ is compact, a finite collection of action sets can be constructed, with arbitrary approximation error. Accordingly, we will assume that the action spaces are finite in the following
Let $\{B_i\}_{i=1}^{M}$ be disjoint subsets of $\mathds{Y}$ such that $\cup_{i=1}^{M}B_i=\mathds{Y}$. This discretization then implies a discretization on the finite-memory and action variables $(h,u)\in(\mathds{Y\times U})^N$.
We denote by $\{A_i\}_{i=1}^{(M\times |\mathds{U}|)^N}$ for the resulting discretization bins of the joint $(h,u)\in(\mathds{Y\times U})^N$ variable.
We define the following basis functions 
\begin{align*}
\phi^i(h,u)=\mathds{1}_{A_i}(h,u), \text{ for all } i=1,\dots, (M\times |\mathds{U}|)^N
\end{align*}
where $\mathds{1}_{A_i}(h,u)$ is the indicator function of the set $A_i$. 

Similar to Section \ref{disc_sec}, we let $q(h)$ denote the quantization map that maps the continuous valued finite-memory variables to its discretized version using the construction in this section. 

 Accordingly, we consider the following iterations, for every $h^i$, $i\in\{1,\dots,M^N\}$, and every $u^j$, $j\in\{1,\dots,|\mathds{U}|^N\}$
\begin{align}\label{opt_iter_pomdp}
Q_{t+1}(h^i,u^j) = Q_t(h^i,u^j) - \alpha_t(h^i,u^j) \left[ Q_t(q(H_t),{U}_t) - c(X_t,U_t) - \beta V_t(q(H_{t+1})) \right]
\end{align}
 where we denote by $V_t(h)=\min_v Q_t(h,v)$:

The following is adapted from   \cite{devran2025near} based on the results in this paper:

\begin{assumption}\label{partial_q}
\hfill
\begin{itemize}
\item [1.] If $(q(H_t),U_t)=(h^i,u^j)$
\[\alpha_t(h^i,u^j) = {1 \over 1+ \sum_{k=0}^{t} 1_{\{q(H_k)=h^i, U_k=u^j\}} }\]
Otherwise $\alpha_t(h,u)=0$.
\item [2.] Under every stationary \{memoryless or finite memory exploration\} policy, say $\gamma$, the true state process, $\{X_t\}_t$, is positive Harris recurrent and in particular admits a unique invariant measure $\pi$.
\item [3.] During the exploration phase, every $(h^i,u^j)$ pair is visited infinitely often.
\item[4.] {$\mathds{Y}\subset\mathds{R}^n$} is compact.
\item[5.] $O(dy|x)=g(x,y)\lambda(dy)$, and $g(y,x)$ is Lipschitz in $y$, such that $|g(x,y)-g(x,y')|\leq\alpha_{\mathds{Y}} \|y-y'\|$ for every $y,y'\in\mathds{Y}$ and $x\in\mathds{X}$ for some $\alpha_{\mathds{Y}}<\infty$.
\item[6.] Stage-wise cost function $c(x,u)$ is bounded such that $\sup_{x,u}c(x,u)=\|c\|_\infty<\infty$.
\end{itemize}
\end{assumption}

\begin{theorem}
\begin{itemize}
\item
Under Assumption \ref{partial_q} for the exploration policy,  the iterations in (\ref{opt_iter_pomdp}) converge to some $Q^*(h,u)$.
\item
Consider the learned policy $\gamma^N$, which satisfies $\gamma^N(h)=\argmin_u Q^*(h,u)$. We assume that the unobserved state initiates at time $-N$ according to some $\mu_{-N}\in \P(\mathds{X})$, and the learned finite-memory policy $\gamma^N$ starts acting at time $t=0$. We denote by $h_0$, the finite-memory variables from time $t=-N$ to $t=0$. For ${z}_0=(\mu_{-N},h_0)$, with a policy $\hat{\gamma}$ acting on the first $N$ steps, we have that
\begin{align*}
&E_{\mu_{-N}}^{\hat{\gamma}}\left[\left|J_\beta({z}_0,{\gamma^N}) - J_\beta^*(z_0)\right|\right]\leq  \frac{2\|c\|_\infty }{(1-\beta)}\sum_{t=0}^\infty\beta^t\hat{L}_t + \frac{\beta}{(1-\beta)^2}\|c\|_\infty \alpha_{\mathds{Y}} L_\mathds{Y} \end{align*}
 where the expectation is with respect to the random realizations of the initial finite-memory variables $h_0$
where
\begin{align*}
 &L_{\mathds{Y}}:= \max_{i}\sup_{y,y'\in B_i}\|y-y'\|,\\
\hat{L}&_t:=\sup_{\hat{\gamma}\in\hat{\Gamma}}E_{\mu}^{\hat{\gamma}}\bigg[\|P^{\pi_t^-}(X_{t+N}\in\cdot|\hat{Y}_{[t,t+N]},U_{[t,t+N-1]})-P^{\pi^*}(X_{t+N}\in\cdot|\hat{Y}_{[t,t+N]},U_{[t,t+N-1]})\|_{TV}\bigg]
\end{align*}
such that the filter stability term $\hat{L}_t$ is with respect to the discretized observations and $\alpha_{\mathds{Y}}$ is the Lipschitz constant of the density function $g$ of the channel $O$.
\end{itemize}
\end{theorem}

\newpage

\appendix

\section{Proof of Lemma \ref{as_conv_lemma}}\label{as_conv_lemma_proof}

\begin{proof}
 We denote by
 \begin{align*}
 e_t(\theta):= \delta^\intercal\left[ A- A(Z_t)\right]\theta + \delta^\intercal\left(b(Z_t)-b\right).
 \end{align*}
 Furthermore, using Assumption \ref{poisson}, we also define
 \begin{align*}
 \phi_t(\theta):=\sum_{k=0}^\infty E\left[e_{t+k}(\theta)|\mathcal{F}_t\right]
 \end{align*}
 We note that under the assumption that $A(z)$ and $b(z)$ are uniformly bounded we have that
 \begin{align*}
 \left| E[\mathds{1}_{\{t+1\leq \sigma_n\}} e_{t+k}(\theta_t)|\mathcal{F}_t] \right| \leq K (2^n+1) \left(\| A- E[A(Z_{t+k})|\mathcal{F}_t]\| +\| b- E[b(Z_{t+k})|\mathcal{F}_t]\|   \right).
 \end{align*}
Using Assumption \ref{poisson}, we know that $\phi_t(\theta)\in L_2$.  Furthermore, we have the following  $L_2$ bound for $\phi_t(\theta)$:
 \begin{align}\label{l2_bound_phi}
 &\|\phi_t(\theta_t)\mathds{1}_{\{t+1\leq \sigma_n\}}\|_2 =  \left\| \sum_{k=0}^\infty E[\mathds{1}_{\{t+1\leq \sigma_n\}} e_{t+k}(\theta_t)|\mathcal{F}_t] \right\|_2\nonumber\\
 &\leq K (2^n+1) \left\| \sum_{k=0}^\infty  \left(\| A- E[A(Z_{t+k})|\mathcal{F}_t]\| +\| b- E[b(Z_{t+k})|\mathcal{F}_t]\|   \right)  \right\|_2\nonumber\\
 &\leq  K (2^n+1) \| Y_t^A + Y_t^b\|_2 < \infty \text{ uniformly for all } t. 
 \end{align}
 where $Y_t^b$ and $Y_t^A$ are defined in (\ref{gordin}), and the last step follows from Assumption \ref{poisson} (ii).

 We write 
 \begin{align*}
E[e_t(\theta)|F_t]& = \phi_t(\theta)- E[\phi_{t+1}(\theta)|F_t]\\
&= \phi_{t+1}(\theta) - E[\phi_{t+1}(\theta)|F_t] + \left(\phi_t(\theta) - \phi_{t+1}(\theta)\right).
\end{align*}
We denote by $\tau^k:=(k+1) \wedge (\sigma_n-1)$.
We assume without generality that $\sigma_n>2$, and write :
\begin{align*}
&\sum_{t=1}^{k+1} \mathds{1}_{\{t+1\leq \sigma_n\}}2\alpha_t \delta_t^\intercal M_t  = \sum_{t=1}^{\tau^k} 2\alpha_t \delta_t^\intercal M_t  = \sum_{t=1}^{\tau^k}2\alpha_t E[e_t(\theta_t)|F_t]\\
& = \sum_{t=1}^{\tau^k}2\alpha_t \left(\phi_{t+1}(\theta_t)-E[\phi_{t+1}(\theta_t)|F_t]\right) + \sum_{t=1}^{\tau^k}2\alpha_t \left(\phi_t(\theta_t) - \phi_{t+1}(\theta_t)\right)\\
& = \sum_{t=1}^{\tau^k}2\alpha_t \left(\phi_{t+1}(\theta_t)-E[\phi_{t+1}(\theta_t)|F_t]\right) \\
& \quad+  \sum_{t=1}^{\tau^k} 2\alpha_t \phi_t(\theta_t) - \sum_{t=0}^{\tau^k-1} 2\alpha_{t+1} \phi_{t+1}(\theta_t)\\
& \quad +  \sum_{t=0}^{\tau^k-1}2\alpha_{t+1} \phi_{t+1}(\theta_t) -   \sum_{t=1}^{\tau^k} 2\alpha_t \phi_{t+1}(\theta_t) \\
& = \sum_{t=1}^{\tau^k}2\alpha_t \left(\phi_{t+1}(\theta_t)-E[\phi_{t+1}(\theta_t)|F_t]\right) \\
&\quad + \sum_{t=1}^{\tau^k} 2\alpha_t \left( \phi_t(\theta_t) -\phi_t(\theta_{t-1}) \right)\\
&\quad +  \sum_{t=1}^{\tau^k-1}2(\alpha_{t+1}-\alpha_t) \phi_{t+1}(\theta_t) \\
&\quad + 2\alpha_1 \phi_1(\theta_0) - 2\alpha_{\tau^k} \phi_{\tau^k+1}(\theta_{\tau^k}) 
\end{align*}
We analyze these terms separately:

{\bf First term:} We first study the term: $\sum_{t=1}^{\tau^k} 2\alpha_t \left(\phi_{t+1}(\theta_t)-E[\phi_{t+1}(\theta_t)|F_t]\right)$. We first note that  $\sum_{t=1}^{k+1} 2\alpha_t  \mathds{1}_{\{t+1\leq \sigma_n\}} \left(\phi_{t+1}(\theta_t)-E[\phi_{t+1}(\theta_t)|F_t]\right)$ is a martingale.
Furthermore, for the increments of this martingale, we have
\begin{align*}
&\sum_t 4\alpha^2_t  E\left[ \mathds{1}_{\{t+1\leq \sigma_n\}}\left(\phi_{t+1}(\theta_t)-E[\phi_{t+1}(\theta_t)|F_t]\right)^2 \right]  \\
&\leq  \sum_t 16\alpha_t^2 E[  \mathds{1}_{\{t+1\leq \sigma_n\}}\phi^2_{t+1}(\theta_t)]=   \sum_t 16\alpha_t^2 \|\mathds{1}_{\{t+1\leq \sigma_n\}}\phi_{t+1}(\theta_t)\|_2^2  \\
& \leq \sum_t 16 K \alpha_t^2 (2^{2n}+1)\|Y^A_{t+1} + Y^b_{t+1}\|^2_2<\infty.
\end{align*}
for some generic constant $K<\infty$, where we used the fact that $\| \mathds{1}_{\{t+1\leq \sigma_n\}}\phi_{t+1}(\theta_t)\|_2 \leq K(2^n+1)\|Y^A_{t+1} + Y^b_{t+1}\|_2$ for some $K<\infty$ following identical steps as in (\ref{l2_bound_phi}).  Furthermore, for the last step, we used the fact that $\sup_t\|Y^A_{t+1} + Y^b_{t+1}\|_2<\infty$ under Assumption \ref{poisson}. We then have a martingale with summable increment variances, and thus    $\sum_{t=1}^{k+1} 2\alpha_t  \mathds{1}_{\{t+1\leq \sigma_n\}} \left(\phi_{t+1}(\theta_t)-E[\phi_{t+1}(\theta_t)|F_t]\right)$ converges a.s..


{\bf Second term:} We now focus on the term $ \sum_{t=1}^{\tau^k} 2\alpha_t \left( \phi_t(\theta_t) -\phi_t(\theta_{t-1}) \right)$. Equivalently, we can study 
\begin{align*}
 \sum_{t=1}^{k+1}\mathds{1}_{\{t+1\leq \sigma_n\}} 2\alpha_t \left( \phi_t(\theta_t) -\phi_t(\theta_{t-1}) \right)
\end{align*}
  Using the fact that $\phi_t \in L_2$ by Assumption \ref{poisson}
\begin{align*}
&\phi_t(\theta_t) -\phi_t(\theta_{t-1}) = \sum_{k=0}^\infty E\left[e_{t+k}(\theta_t) - e_{t+k}(\theta_{t-1}) |F_t\right] \\
&= \sum_{k=0}^\infty E\left[ (\delta_t - \delta_{t-1})^\intercal [A-A(Z_{t+k})] \theta_{t-1} + \delta_t^\intercal[A-A(Z_{t+k})](\theta_t - \theta_{t-1})   |F_t\right]
\end{align*}
We note that on the event $t\leq \sigma_n$, using the boundedness of $A,A(Z_t),b,b(Z_t)$ we have that 
\begin{align*}
&\delta_t - \delta_{t-1} = \alpha_{t-1}(-A\delta_t + M_t) \leq \alpha_{t-1} K (1+2^{\frac{n}{2}})\\
&\theta_t-\theta_{t-1} = \alpha_{t-1}(-A(Z_t)\theta_t + b(Z_t))\leq \alpha_{t-1} K (1+2^{\frac{n}{2}}).
\end{align*}
Using these bounds, and following the identical steps as in (\ref{l2_bound_phi}),  and by Assumption \ref{poisson}, we can then write for some generic constant $K<\infty$ that
\begin{align*}
&\| \mathds{1}_{\{t+1\leq \sigma_n\}}\left(\phi_t(\theta_t) -\phi_t(\theta_{t-1})\right)\|_2 \leq \alpha_{t-1} K (1+2^n)\|Y_t^A+Y_t^b\|_2.
\end{align*}
Consequently, we write 
\begin{align*}
&\lim_{k\to\infty}E\left[  \sum_{t=1}^{k+1}\left|\mathds{1}_{\{t+1\leq \sigma_n\}} 2\alpha_t \left( \phi_t(\theta_t) -\phi_t(\theta_{t-1}) \right)\right| \right]\\
&=E\left[  \sum_{t=1}^{\infty}\left|\mathds{1}_{\{t+1\leq \sigma_n\}} 2\alpha_t \left( \phi_t(\theta_t) -\phi_t(\theta_{t-1}) \right)\right| \right]\\
&=\sum_{t=1}^{\infty} 2\alpha_t E\left[  \left|\mathds{1}_{\{t+1\leq \sigma_n\}}\left( \phi_t(\theta_t) -\phi_t(\theta_{t-1}) \right)\right| \right]\\
&\leq \sum_{t=1}^{\infty} 2\alpha_t \alpha_{t-1} K (1+2^n)\|Y_t^A+Y_t^b\|_2 <\infty
\end{align*}
where we used the uniform boundedness of $\|Y_t^A+Y_t^b\|_2 $ over $t$ at the last step. We can then conclude that $ \sum_{t=1}^{k+1}\left|\mathds{1}_{\{t+1\leq \sigma_n\}} 2\alpha_t \left( \phi_t(\theta_t) -\phi_t(\theta_{t-1}) \right)\right|<\infty$  almost surely and thus $ \sum_{t=1}^{k+1}\mathds{1}_{\{t+1\leq \sigma_n\}} 2\alpha_t \left( \phi_t(\theta_t) -\phi_t(\theta_{t-1}) \right)$  converges almost surely.


{\bf Third term:} We now study the term $ \sum_{t=1}^{\tau^k-1} (\alpha_{t+1}-\alpha_t) \phi_{t+1}(\theta_t) $. 
\begin{align*}
  &E\left[\lim_{k\to\infty}\sum_{t=1}^{\tau^k-1} \left|(\alpha_{t+1}-\alpha_t) \phi_{t+1}(\theta_t)\right| \right]\\
   &= E\left[\lim_{k\to\infty}\sum_{t=1}^{k} \mathds{1}_{\{t+1\leq\sigma_n-1\}}\left|(\alpha_{t+1}-\alpha_t) \phi_{t+1}(\theta_t)\right| \right]\\
 &= E\left[\sum_{t=1}^{\infty} \mathds{1}_{\{t+1\leq\sigma_n-1\}}\left|(\alpha_{t+1}-\alpha_t) \phi_{t+1}(\theta_t)\right| \right]\\
 & \leq\sum_{t=1}^\infty (\alpha_t-\alpha_{t+1}) E\left[ \mathds{1}_{\{t+1\leq\sigma_n\}}\left|\phi_{t+1}(\theta_t)\right| \right]\\
 &\leq \sum_{t=1}^\infty (\alpha_t-\alpha_{t+1}) \|\mathds{1}_{\{t+1\leq\sigma_n\}}\phi_{t+1}(\theta_t)\|_2\\
 &\leq K (2^n+1)\sum_{t=1}^\infty (\alpha_t-\alpha_{t+1}) = K (2^n+1)\alpha_1
\end{align*}
and thus $\sum_{t=1}^{\tau^k-1} (\alpha_{t+1}-\alpha_t) \phi_{t+1}(\theta_t) $ converges almost surely as $k\to\infty$.

{\bf Last term:} Finally, $2\alpha_1 \phi_1(\theta_0) - 2\alpha_{\tau^k} \phi_{\tau^k+1}(\theta_{\tau^k})$, we have that  
\begin{align*}
 2\alpha_{\tau^k} \phi_{\tau^k+1}(\theta_{\tau^k}) \to \begin{cases}
 \alpha_{\sigma_n-1}\phi_{\sigma_n-1}(\theta_{\sigma_n-1}) &\text{if $ \sigma_n<\infty$ }\\
            \lim_{k\to\infty}\alpha_{k+1}\phi_{k+2}(\theta_{k+1})=0, & \text{if $\sigma_n=\infty$}
 \end{cases}
\end{align*}
For the last part, using similar arguments as before, we can show that 
\begin{align*}
E[\sum_{k=0}^\infty \mathds{1}_{\{ k+1\leq \sigma_n\}}  \left( \alpha_{k} \phi_{k+1}(\theta_{k})\right)^2 ]<\infty
\end{align*}
which then implies that on $\{\sigma_n=\infty\}$, $\sum_{k=0}^\infty   \left( \alpha_{k} \phi_{k+1}(\theta_{k})\right)^2 <\infty$ almost surely, and that $ \alpha_{k} \phi_{k+1}(\theta_{k})\to 0$ almost surely. 

{\bf Final step:} So far we have shown that
$
\sum_{t=1}^{k+1} \mathds{1}_{\{t+1\leq \sigma_n\}}2\alpha_t \delta_t^\intercal M_t 
$
converges almost surely.  This then immediately implies that $ \sum_{t=1}^{k+1} 2\alpha_t \delta_t^\intercal M_t $ converges almost surely on the event $\sigma_n=\infty$ since $\mathds{1}_{\{t+1\leq \sigma_n\}}=1$ on  $\sigma_n=\infty$  for all $t$.

\end{proof}

\section{Proof of Lemma \ref{L2_bound_lemma}}\label{L2_bound_lemma_proof}

We have that for any $k>n$:
\begin{align*}
 \mathds{1}_{\{k+1 \leq \sigma(C)\}}\left(\sum_{t=n}^{k} 2\alpha_t \delta_t^\intercal M_t\right)^2\leq \left(\sum_{t=n}^{k} \mathds{1}_{\{t+1\leq \sigma(C)\}}2\alpha_t \delta_t^\intercal M_t\right)^2.
\end{align*}
Furthermore,  denoting by $\tau^k:=k \wedge (\sigma(C)-1)$. we have that 
\begin{align*}
&E\left[\sup_{k>n} \left(\sum_{t=n}^{k} \mathds{1}_{\{t+1\leq \sigma(C)\}}2\alpha_t \delta_t^\intercal M_t\right)^2  \right]=E\left[\sup_{k>n} \mathds{1}_{\{\sigma(C)>n\}} \left(\sum_{t=n}^{\tau^k} 2\alpha_t \delta_t^\intercal M_t\right)^2  \right]
\end{align*}

We then write
\begin{align*}
& \sum_{t=n}^{\tau^k} 2\alpha_t \delta_t^\intercal M_t  = \sum_{t=n}^{\tau^k}2\alpha_t E[e_t(\theta_t)|F_t]\\
& = \sum_{t=n}^{\tau^k}2\alpha_t \left(\phi_{t+1}(\theta_t)-E[\phi_{t+1}(\theta_t)|F_t]\right) + \sum_{t=n}^{\tau^k}2\alpha_t \left(\phi_t(\theta_t) - \phi_{t+1}(\theta_t)\right)\\
& = \sum_{t=n}^{\tau^k}2\alpha_t \left(\phi_{t+1}(\theta_t)-E[\phi_{t+1}(\theta_t)|F_t]\right) \\
& \quad+  \sum_{t=n}^{\tau^k} 2\alpha_t \phi_t(\theta_t) - \sum_{t=n-1}^{\tau^k-1} 2\alpha_{t+1} \phi_{t+1}(\theta_t)\\
& \quad +  \sum_{t=n-1}^{\tau^k-1}2\alpha_{t+1} \phi_{t+1}(\theta_t) -   \sum_{t=n}^{\tau^k} 2\alpha_t \phi_{t+1}(\theta_t) \\
& = \sum_{t=n}^{\tau^k}2\alpha_t \left(\phi_{t+1}(\theta_t)-E[\phi_{t+1}(\theta_t)|F_t]\right) \\
&\quad + \sum_{t=n}^{\tau^k} 2\alpha_t \left( \phi_t(\theta_t) -\phi_t(\theta_{t-1}) \right)\\
&\quad +  \left(\sum_{t=n}^{\tau^k-1}2(\alpha_{t+1}-\alpha_t) \phi_{t+1}(\theta_t)\right) \mathds{1}_{\{ \sigma(C)>n+1\}} \\
&\quad + 2\alpha_n \phi_n(\theta_{n-1}) - 2\alpha_{\tau^k} \phi_{\tau^k+1}(\theta_{\tau^k}) 
\end{align*}
We analyze these terms separately:

{\bf First term:} For the first term, we first recall that for any $k> n$ $$ Z_n^k:=\sum_{t=n}^{k} \mathds{1}_{\{ t+1\leq \sigma(C)\}}2\alpha_t \left(\phi_{t+1}(\theta_t)-E[\phi_{t+1}(\theta_t)|F_t]\right)$$ is a martingale sequence.  Following the same steps as in Lemma \ref{as_conv_lemma} we have that
\begin{align*}
&E[(Z_n^k)^2]=E\left[ \left( \sum_{t=n}^{k} \mathds{1}_{\{ t+1\leq \sigma(C)\}}2\alpha_t \left(\phi_{t+1}(\theta_t)-E[\phi_{t+1}(\theta_t)|F_t]\right)\right)^2\right]\\
&= \sum_{t=n}^{k} 4\alpha^2_t E\left[\mathds{1}_{\{ t+1\leq \sigma(C)\}} \left(\phi_{t+1}(\theta_t)-E[\phi_{t+1}(\theta_t)|F_t]\right)^2\right]\\
&= \sum_{t=n}^{k} 16 K \alpha^2_t (C^2 + 1)\|Y_t^A + Y_t^b\|_2^2  \\
&\leq 16K (C^2 + 1)\left(\sup_t\|Y_t^A + Y_t^b\|_2^2 \right)\sum_{t=n}^{k}  \alpha^2_t = K' (C^2 + 1) \sum_{t=n}^{k}  \alpha^2_t
\end{align*}
for some $K,K'<\infty$. Hence,  using Doob's maximal inequality, together with the monotone convergence theorem we can write that 
\begin{align*}
&E[\sup_{n<k}|Z_n^k|^2] = \lim_{N\to\infty}E\left[ \sup_{n<k<N}|Z_n^k|^2\right] \\
& \leq 4  \sup_{n<k<N} E[|Z_n^k|^2] \leq 4K' (C^2 + 1) \sum_{t=n}^{\infty}  \alpha^2_t.
\end{align*}

We can then write
\begin{align*}
&E\left[\sup_{k>n} \mathds{1}_{\{\sigma(C)>n\}} \left( \sum_{t=n}^{\tau^k}2\alpha_t \left(\phi_{t+1}(\theta_t)-E[\phi_{t+1}(\theta_t)|F_t]\right)\right)^2\right]\\
&= E\left[\sup_{k>n}  \left( \sum_{t=n}^{k} \mathds{1}_{\{ t+1\leq \sigma_n\}}2\alpha_t \left(\phi_{t+1}(\theta_t)-E[\phi_{t+1}(\theta_t)|F_t]\right)\right)^2\right]\\
&= E\left[\sup_{k>n} |Z_n^k|^2  \right]\leq  4K' (C^2 + 1) \sum_{t=n}^{\infty}  \alpha^2_t.
\end{align*}

{\bf Second term:} We follow the same steps as in Lemma \ref{as_conv_lemma} and write 
\begin{align*}
&E\left[\sup_{k>n} \mathds{1}_{\{\sigma(C)>n\}} \left( \sum_{t=n}^{\tau^k}2\alpha_t \left(\phi_{t}(\theta_t)-\phi_{t}(\theta_{t-1})\right)\right)^2\right]\\
&=E\left[\sup_{k>n}  \left( \sum_{t=n}^{k}\mathds{1}_{\{t+1\leq\sigma(C)\}}2\alpha_t \left(\phi_{t}(\theta_t)-\phi_{t}(\theta_{t-1})\right)\right)^2\right]\\
&\leq E\left[\sup_{k>n}  \left( \sum_{t=n}^{k}4\alpha^2_t\right) \left( \sum_{t=n}^{k} \mathds{1}_{\{t+1\leq\sigma(C)\}}\left( \phi_{t}(\theta_t)-\phi_{t}(\theta_{t-1})\right)^2\right)\right]\\
&\leq    \sum_{t=n}^\infty 4\alpha_t^2 \sum_{t=n}^\infty E\left[ \mathds{1}_{\{ t+1\leq \sigma(C)\}} (\phi_t(\theta_t) - \phi_t(\theta_{t-1}))^2 \right]\\
&=  \sum_{t=n}^\infty4 \alpha_t^2 \sum_{t=n}^\infty \| \mathds{1}_{\{t+1\leq \sigma(C)\}}\left(\phi_t(\theta_t) -\phi_t(\theta_{t-1})\right)\|^2_2 \\
&\leq  \sum_{t=n}^\infty 4\alpha_t^2 \sum_{t=n}^\infty   \alpha^2_{t-1} (1+C^2)\|Y_t^A+Y_t^b\|^2_2\\
&\leq K (1+C^2) \sum_{t=n}^\infty \alpha_t^2 \sum_{t=n}^\infty   \alpha^2_{t-1} 
\end{align*}
for some constant $K<\infty$.

{\bf Third Term}
We use the Cauchy-Schwartz Theorem and that $\mathds{1}_{\{t+1\leq \sigma(C)-1\}} \leq \mathds{1}_{\{t+1\leq \sigma(C)\}} $ to write \begin{align*}
&E\left[\sup_{k>n} \mathds{1}_{\{\sigma(C)>n+1\}} \left( \sum_{t=n}^{\tau^k-1}2(\alpha_{t+1}-\alpha_t) \phi_{t+1}(\theta_t)\right)^2\right]\\
&\leq E\left[\sup_{k>n}  \left( \sum_{t=n}^{k-1} \mathds{1}_{\{t+1\leq \sigma(C)-1\}} 2(\alpha_{t+1}-\alpha_t) \phi_{t+1}(\theta_t)\right)^2  \right]\\
&\leq E\left[\sup_{k>n} \left(\sum_{t=n}^{k-1}  4(\alpha_{t}-\alpha_{t+1})\right) \left(\sum_{t=n}^{k-1} \mathds{1}_{\{t+1\leq \sigma(C)\}} (\alpha_{t}-\alpha_{t+1})\ \phi_{t+1}(\theta_t)^2  \right)\right]\\
&\leq \sum_{t=n}^{\infty}  4(\alpha_{t}-\alpha_{t+1}) \sum_{t=n}^{\infty} (\alpha_{t}-\alpha_{t+1}) \left\| \mathds{1}_{\{t+1\leq \sigma(C)\}} \phi_{t+1}(\theta_t) \right\|_2^2\\
&\leq K(C^2+1)\sum_{t=n}^{\infty}  (\alpha_{t}-\alpha_{t+1}) \sum_{t=n}^{\infty}   (\alpha_{t}-\alpha_{t+1}) \\
&\leq  K(C^2+1)\alpha_n^2
\end{align*}
where we used the fact that $\left\| \mathds{1}_{\{t+1\leq \sigma(C)\}} \phi_{t+1}(\theta_t) \right\|_2^2 \leq (C^2+1) K \|Y_t^A+Y_t^b\|_2^2$ and that $\sup_t  \|Y_t^A+Y_t^b\|_2^2 < \infty$ by assumption.

{\bf The last term:} 

\begin{align*}
&E\left[ \sup_{k<n}\mathds{1}_{\{\sigma(C)>n\}}\left(2\alpha_n \phi_n(\theta_{n-1}) - 2\alpha_{\tau^k} \phi_{\tau^k+1}(\theta_{\tau^k})\right)^2 \right]\\
&\leq 4\alpha_n^2 E\left[\mathds{1}_{\{n+1\leq \sigma(C)\}}\phi_n(\theta_{n-1})^2\right] + E\left[ \sup_{k<n}\mathds{1}_{\{\sigma(C)>n\}} \sum_{t=n}^{\tau^k}(\alpha_t \phi_{t+1}(\theta_t))^2\right]\\
&\leq  4\alpha_n^2 \|\mathds{1}_{\{n\leq \sigma(C)\}}\phi_n(\theta_{n-1})\|_2^2+  E\left[ \sup_{k<n}\sum_{t=n}^k \mathds{1}_{\{t+1\leq \sigma(C)\}}\alpha_t^2 \phi_{t+1}(\theta_t)^2\right]\\
&\leq K \alpha_n^2 (1+C^2) + \sum_{t=n}^\infty \alpha_t^2 \|  \mathds{1}_{\{t+1\leq \sigma(C)\}}  \phi_{t+1}(\theta_t)\|_2^2 \\
&\leq K \alpha_n^2 (1+C^2) + K  (1+C^2) \sum_{t=n}^\infty \alpha_t^2
\end{align*}

\section{Proof of Lemma \ref{bound_lemma}}\label{bound_lemma_proof}

\begin{proof}
We introduce the following stopping times ($\sigma_n$ has been introduced earlier in (\ref{stop_time})):
\begin{align*}
&\sigma_n:=\inf \{t: \|\delta_t\|^2 > 2^n\}\\
& \tau_n:=1+\sup\{t<\sigma_{n+1}: \|\delta_t\|^2\leq 2^n\}.
\end{align*}
Using the bound on $\|M_t\|$ such that $\|M_t\| \leq K(\|\delta_t\|+1)$ for some $K<\infty$, we can write 
\begin{align*}
\|\delta_{t+1}\|^2-\|\delta_t\|^2\leq \alpha_t K (1+\|\delta_t\|^2) + K \alpha_t^2(1+\|\delta_t\|^2)
\end{align*}
for some generic constant $K<\infty$. If we define the set 
\begin{align*}
C_n:=\{ \forall t\geq n: \|\delta_{t+1}\|^2-\|\delta_t\|^2 \leq \frac{1}{2}(\|\delta_t\|^2+1) \}
\end{align*}
then there exists some $r<\infty$ such that $P(C_n)=1$ for all $r\geq n$.

On $C_r$, we have that 
$
\|\delta_{t+1}\|^2 + 1 \leq \frac{3}{2}(\|\delta_t\|^2+1)
$ for all $t\geq r$, it then follows that for all $t\geq r$, $\|\delta_{t}\|^2 + 1 \leq \frac{3}{2}^{t-r}(\|\delta_r\|^2+1)$. Consider 
\begin{align*}
\{\sigma_n\leq n\} = \{\sup_{r\leq t\leq n} \|\delta_t\|^2 \geq 2^n\} \cup \cup_{t=1}^{r-1} \{\|\delta_t\|^2\geq 2^n\}.
\end{align*}
Note that $\lim_{n\to\infty}P(\|\delta_t\|^2\geq 2^n)=0$ for every fixed $t<r$ using the bounds on $A(Z_t)$ and $b(Z_t)$. We then have that
\begin{align*}
&\lim_{n\to\infty}P(C_r \cap (\sigma_n\leq n)) = \lim_{n\to\infty}P(C_r\cap (\sup_{r\leq t\leq n} \|\delta_t\|^2\geq 2^n))\\
&\leq \lim_{n\to\infty} P(\|\delta_r\|^2 + 1 \geq 2^n \frac{3}{2}^{r-n}) = 0.
\end{align*}
Since, $P(C_r)= 1$, we then have that 
\begin{align*}
\lim_{n\to\infty}P(n<\sigma_n)=1.
\end{align*}
We now define
\begin{align}\label{B_n_set}
B_n:=C_n\cap (n<\sigma_n)
\end{align}
such that $P(B_n)\to 1$. Note that on $\{\sigma_{n+1}<\infty\}$, $\|\delta_{\sigma_{n+1}}\|^2 \geq 2^{n+1}$. Furthermore, on $B_n$, $\tau_{n}\geq \sigma_n>n$, and we have that $\|\delta_{(\tau_n-1)}\|^2\leq 2^n$. We then write,
\begin{align*}
\|\delta_{\tau_n}\|^2\leq \frac{3}{2}\|\delta_{(\tau_n-1)}\|^2 + \frac{1}{2}\leq \frac{3}{2}2^n + \frac{1}{2}.
\end{align*}
It then follows that on $B_n\cap (\sigma_{n+1}<\infty)$
\begin{align}\label{key_step1}
\|\delta_{\sigma_{n+1}}\|^2 - \|\delta_{\tau_n}\|^2\geq 2^{n+1} - \frac{3}{2}2^n - \frac{1}{2}\geq \frac{2^n}{4}
\end{align}
for all $n\geq 2$.

We now focus on the upper bound. Using the iterative form in (\ref{robbin_bound}), on $B_n\cap (\sigma_{n+1}<\infty)$ we have that 
\begin{align*}
\|\delta_{\sigma_{n+1}}\|^2 - \|\delta_{\tau_n}\|^2& \leq \sum_{t=\tau_n}^{\sigma_{n+1}-1} K \alpha_t^2 \|\delta_t\|^2 + \sum_{t=\tau_n}^{\sigma_{n+1}-1}  2\alpha_t\delta_t^\intercal M_t\\
&\leq 2^{n+1}\sum_{t=n}^{\infty} K \alpha_t^2 + \sup_{n<t<\sigma_{n+1}} \left| \sum_{k=t}^{\sigma_{n+1}-1} 2\alpha_k\delta_k^\intercal M_k\right|\\
&\leq  2^{n+1}\sum_{t=n}^{\infty} K \alpha_t^2 + 2\sup_{n<t\leq \sigma_{n+1}}  \left|\sum_{k=n}^{t-1} 2\alpha_k\delta_k^\intercal M_k\right|\\
&\leq  2^{n+1}\sum_{t=n}^{\infty} K \alpha_t^2 + 2\sup_{n<t} \mathds{1}_{\{t\leq \sigma_{n+1}\}}  \left| \sum_{k=n}^{t-1}2\alpha_k\delta_k^\intercal M_k\right|
\end{align*}
By Lemma \ref{L2_bound_lemma}, we have that 
\begin{align*}
&E\left[\sup_{t>n} \mathds{1}_{\{t \leq \sigma_{n+1}\}}\left(\sum_{k=n}^{t-1} 2\alpha_k \delta_k^\intercal M_k\right)^2  \right] \\
& = E\left[\sup_{t>n-1} \mathds{1}_{\{t+1 \leq \sigma_{n+1}\}}\left(\sum_{k=n}^{t} 2\alpha_k \delta_k^\intercal M_k\right)^2  \right] \\
&\leq E\left[\sup_{t>n} \mathds{1}_{\{t+1 \leq \sigma_{n+1}\}}\left(\sum_{k=n}^{t} 2\alpha_k \delta_k^\intercal M_k\right)^2  \right] +  E\left[ \mathds{1}_{\{n+1 \leq \sigma_{n+1}\}}\left( 2\alpha_n \delta_n^\intercal M_n\right)^2  \right] \\
&\leq K (1+2^{2n+2}) \sum_{k=n}^{\infty} \alpha_k^2 + K (1+2^{2n+2}) \alpha_n^2\leq  K (1+2^{2n+2}) \sum_{k=n}^{\infty} \alpha_k^2 
\end{align*}
where we used a generic $K<\infty$ which might change at different steps.
It then follows that 
\begin{align*}
&E\left[\mathds{1}_{B_n\cap(\sigma_{n+1}<\infty)} (\|\delta_{\sigma_{n+1}}\|^2 - \|\delta_{\tau_n}\|^2)^2 \right] \leq K 2^{2n} \left(\sum_{t=n}^{\infty}  \alpha_t^2 \right)^2+K 2^{2n} \sum_{t=n}^{\infty}  \alpha_t^2 
\end{align*}
for some constant $K<\infty$. Combining this bound, with (\ref{key_step1}), we can write
\begin{align*}
K \left(\sum_{t=n}^{\infty}  \alpha_t^2 \right)^2 + K\sum_{t=n}^{\infty}  \alpha_t^2& \geq 2^{-2n}E\left[\mathds{1}_{B_n\cap(\sigma_{n+1}<\infty)} (\|\delta_{\sigma_{n+1}}\|^2 - \|\delta_{\tau_n}\|^2)^2 \right]\\
& \geq \frac{1}{16} P(B_n\cap(\sigma_{n+1}<\infty)).
\end{align*}
Noting that $P(B_n)\to 1$ (see (\ref{B_n_set})), and that $ \sum_{t=n}^{\infty}  \alpha_t^2 \to 0$, we then conclude that 
$
P(\sigma_{n}<\infty)\to 0.
$
\end{proof}

\bibliographystyle{plain}

\bibliography{AliBibliography,references_acc,SerdarBibliography_acc,references}

\end{document}